\documentclass{article}

\usepackage{hyperref}
\usepackage{amsmath,amssymb,amsthm}
\usepackage{algorithm}
\usepackage{algorithmic}
\usepackage{xspace}
\usepackage{graphicx}
\usepackage{setspace}
\usepackage{subfigure}
\usepackage{fullpage}

\newcommand{\alg}{\textsc{Torrent}\xspace}
\newcommand{\ght}{\textup{\textrm{HT}}\xspace}
\newcommand{\fc}{\textup{\textrm{FC}}}
\newcommand{\gd}{\textup{\textrm{GD}}}

\newcommand{\bt}{\w}
\newcommand{\bto}{\bt^\ast}
\newcommand{\btt}{\bt^t}
\newcommand{\btn}{\bt^{t+1}}
\newcommand{\btnn}{\bt^{t+2}}

\newcommand{\bth}{\hat\bt}

\newcommand{\fa}{\text{CR}}
\newcommand{\md}{\text{MD}}

\newcommand{\<}{\leftarrow}

\newcommand{\normsg}[1]{\norm{#1}_{\psi_2}}

\newcommand{\R}{{\mathbb R}}

\newcommand{\cS}{{\mathcal S}}
\newcommand{\cC}{{\mathcal C}}

\newcommand{\cN}{{\mathcal N}}

\newcommand{\g}{{\mathbf g}}
\renewcommand{\r}{{\mathbf r}}
\renewcommand{\u}{{\mathbf u}}
\renewcommand{\b}{{\mathbf b}}

\newcommand{\z}{{\mathbf z}}

\renewcommand{\vec}[1]{{\mathbf{#1}}}

\newcommand{\veczero}{\vec{0}}
\newcommand{\vz}{\veczero}

\newcommand{\x}{\vec{x}}

\newcommand{\w}{\vec{w}}
\renewcommand{\v}{\vec{v}}
\newcommand{\y}{\vec{y}}

\newcommand{\ve}{\varepsilon}
\newcommand{\vve}{\text{\boldmath{$\ve$}}}

\newcommand{\bc}[1]{\left\{{#1}\right\}}
\newcommand{\br}[1]{\left({#1}\right)}
\newcommand{\bs}[1]{\left[{#1}\right]}
\newcommand{\abs}[1]{\left| {#1} \right|}
\newcommand{\norm}[1]{\left\| {#1} \right\|}

\renewcommand{\O}[1]{{\cal O}\br{{#1}}}
\newcommand{\softO}[1]{\widetilde{\cal O}\br{{#1}}}
\newcommand{\Om}[1]{\Omega\br{{#1}}}
\newcommand{\softOm}[1]{\tilde\Omega\br{{#1}}}

\newcommand{\Ebb}{{\mathbb E}}

\renewcommand{\Pr}[1]{{{\mathbb P}}\bs{{#1}}}

\newcommand{\ip}[2]{\left\langle{#1},{#2}\right\rangle}

\newtheorem{lem}{Lemma}
\newtheorem{thm}[lem]{Theorem}
\newtheorem{cor}[lem]{Corollary}
\newtheorem{clm}[lem]{Claim}
\newtheorem{defn}[lem]{Definition}

\theoremstyle{remark}
\newtheorem{remark}{Remark}

\newcommand{\Xt}{\widetilde{X}}

\newcommand{\isS}{\Sigma^{-1/2}}
\newcommand{\isSo}{\Sigma_0^{-1/2}}
\newcommand{\sSo}{\Sigma_0^{1/2}}

\makeatletter
\newcommand{\newreptheorem}[2]{\newtheorem*{rep@#1}{\rep@title} 
\newenvironment{rep#1}[1]{\def\rep@title{#2 \ref*{##1}}\begin{rep@#1}}{\end{rep@#1}}
}
\makeatother

\newcommand\epigraph[5]{
\vspace{1em}\hfill{}\begin{minipage}{#1}{\begin{spacing}{1}
\noindent\textit{#2}\end{spacing}
\vspace{1em}
\hfill{}\textup{#3}, \textit{#4} \textup{#5}.}
\end{minipage}}

\newreptheorem{lem}{Lemma}
\newreptheorem{thm}{Theorem}
\newreptheorem{clm}{Claim}

\title{Robust Regression via Hard Thresholding}

\author{Kush Bhatia \qquad Prateek Jain \qquad Purushottam Kar\\Microsoft Research India\\\texttt{\{t-kushb,prajain,t-purkar\}@microsoft.com}}

\date{}

\begin{document}

\maketitle
\begin{abstract}
We study the problem of Robust Least Squares Regression (RLSR) where several response variables can be adversarially corrupted. More specifically, for a data matrix $X\in \R^{p\times n}$ and an underlying model $\bto$, the response vector is generated as $\y=X^T\bto+\b$ where $\b\in \R^n$ is the corruption vector supported over at most $C\cdot n$ coordinates. Existing exact recovery results for RLSR focus solely on $L_1$-penalty based convex formulations and impose relatively strict model assumptions such as requiring the corruptions $\b$ to be selected independently of $X$.

In this work, we study a simple hard-thresholding algorithm called \alg which, under mild conditions on $X$, can recover $\bto$ exactly even if $\b$ corrupts the response variables in an \emph{adversarial} manner, i.e. both the support and  entries of $\b$ are selected adversarially after observing $X$ and $\bto$. Our results hold under  \emph{deterministic} assumptions which are satisfied if $X$ is sampled from any sub-Gaussian distribution. Finally unlike existing results that apply only to a fixed $\bto$, generated independently of $X$, our results are \emph{universal} and hold for any $\bto\in \R^p$.

Next, we propose gradient descent-based extensions of \alg that can scale efficiently to large scale problems, such as high dimensional sparse recovery and prove similar recovery guarantees for these extensions. Empirically we find \alg, and more so its extensions, offering significantly faster recovery than the state-of-the-art $L_1$ solvers. For instance, even on moderate-sized datasets (with $p=50K$) with around $40\%$ corrupted responses, a variant of our proposed method called \alg-HYB is more than $20\times$ faster than the best $L_1$ solver. 

\end{abstract}


\vskip2ex

\epigraph{0.67\linewidth}{``If among these errors are some which appear too large to be admissible, then those equations which produced these errors will be rejected, as coming from too faulty experiments, and the unknowns will be determined by means of the other equations, which will then give much smaller errors."\vspace*{-8pt}}{A. M. Legendre}{On the Method of Least Squares.}{1805}

\section{Introduction}
Robust Least Squares Regression (RLSR) addresses the problem of learning a reliable set of regression coefficients in the presence of several arbitrary corruptions in the {\em response} vector. Owing to the wide-applicability of regression, RLSR features as a critical component of several important real-world applications in a variety of domains such as signal processing \cite{StuderKPB12}, economics \cite{RousseeuwL87}, computer vision \cite{WrightYGSM09,WrightM10}, and astronomy \cite{RousseeuwL87}.

Given a data matrix $X=\bs{\x_1, \dots, \x_n}$ with $n$ data points in $\R^p$ and the corresponding response vector $\y\in \R^n$, the goal of RLSR is to learn a $\bth$ such that,
\begin{equation}\label{eq:rlsr}
(\bth, \hat S)=\underset{\substack{\bt \in \R^p\\S \subset [n]: |S|\geq (1-\beta)\cdot n}}{\arg\min}\sum_{i\in S}(y_i-\x_i^T\bt)^2,
\end{equation}
That is, we wish to simultaneously determine the set of corruption free points $\hat S$ and also estimate the best model parameters over the set of clean points. However, the optimization problem given above is non-convex (jointly in $\bt$ and $S$) in general and might not directly admit efficient solutions. Indeed there exist reformulations of this problem that are known to be NP-hard to optimize \cite{StuderKPB12}.

To address this problem, most existing methods with provable guarantees assume that the observations are obtained from some generative model. A commonly adopted model is the following
\begin{equation}
  \label{eq:modely}
  \y=X^T\bto+\b,
\end{equation}
where $\bto \in \R^p$ is the \emph{true} model vector that we wish to estimate and $\b\in \R^n$ is the corruption vector that can have arbitrary values. A common assumption is that the corruption vector is \emph{sparsely} supported i.e. $\|\b\|_0\leq \alpha\cdot n$ for some $\alpha > 0$.

Recently, \cite{WrightM10} and \cite{NguyenT13} obtained a surprising result which shows that one can recover $\bto$ {\em exactly} even when $\alpha \lesssim 1$, i.e., when almost all the points are corrupted, by solving an $L_1$-penalty based convex optimization problem: $\min_{\bt,\b}\norm{\bt}_1+\lambda \norm{\b}_1$, s.t., $X^\top\bt+\b=\y$. However, these results require the corruption vector $\b$ to be selected oblivious of $X$ and $\bto$. Moreover, the results impose severe restrictions on the data distribution, requiring that the data be either sampled from an isotropic Gaussian ensemble \cite{WrightM10}, or row-sampled from an incoherent orthogonal matrix \cite{NguyenT13}. Finally, these results hold only for a fixed $\bto$ and are not universal in general. 

In contrast, \cite{ChenCM13} studied RLSR with less stringent assumptions, allowing arbitrary corruptions in response variables as well as in the data matrix $X$, and proposed a trimmed inner product based algorithm for the problem. However, their recovery guarantees are significantly weaker. Firstly, they are able to recover $\bto$ only upto an additive error $\alpha\sqrt{p}$ (or $\alpha\sqrt{s}$ if $\bto$ is $s$-sparse). Hence, they require $\alpha\le 1/\sqrt{p}$ just to claim a non-trivial bound. Note that this amounts to being able to tolerate only a vanishing fraction of corruptions. More importantly, even with $n\rightarrow \infty$ and extremely small $\alpha$ they are unable to guarantee exact recovery of $\bto$. A similar result was obtained by \cite{McWilliamsKLB14}, albeit using a sub-sampling based algorithm with stronger assumptions on $\b$. 

In this paper, we focus on a simple and natural thresholding based algorithm for RLSR. At a high level, at each step $t$, our algorithm alternately estimates an \emph{active set} $S_t$ of ``clean'' points and then updates the model to obtain $\btn$ by minimizing the least squares error on the active set. This intuitive algorithm seems to embody a long standing heuristic first proposed by Legendre \cite{Legendre59} over two centuries ago (see introductory quotation in this paper) that has been adopted in later literature \cite{Rousseeuw84, RousseeuwD06} as well. However, to the best of our knowledge, this technique has never been rigorously analyzed before in non-asymptotic settings, despite its appealing simplicity.

\textbf{Our Contributions}: The main contribution of this paper is an exact recovery guarantee for the thresholding algorithm mentioned above that we refer to as \alg-FC (see Algorithm~\ref{algo:fcht}). We provide our guarantees in the model given in~\ref{eq:modely} where the corruptions $\b$ are selected {\em adversarially} but restricted to have at most $\alpha \cdot n$ non-zero entries where $\alpha < 1/2$ is a global constant dependent only on $X$\footnote{Note that for an adaptive adversary, as is the case in our work, recovery cannot be guaranteed for $\alpha \geq 1/2$ since the adversary can introduce corruptions as $\b_i = \x_i^\top(\widetilde\bt - \bto)$ for an adversarially chosen model $\widetilde\bt$. This would make it impossible for any algorithm to distinguish between $\bto$ and $\widetilde\bt$ thus making recovery impossible.}. Under \emph{deterministic} conditions on $X$, namely the subset strong convexity (SSC) and smoothness (SSS) properties (see Definition~\ref{defn:ssc-sss}), we guarantee that \alg-FC converges at a \emph{geometric} rate and recovers $\bto$ exactly. We further show that these properties (SSC and SSS) are satisfied w.h.p. if a) the data $X$ is sampled from a sub-Gaussian distribution and, b) $n \geq p\log p$.

We would like to stress three key advantages of our result over the results of \cite{WrightM10,NguyenT13}: a) we allow $\b$ to be adversarial, i.e., both support and values of $\b$ to be selected adversarially based on $X$ and $\bto$, b) we make assumptions on data that are natural, as well as significantly less restrictive than what existing methods make, and c) our analysis admits universal guarantees, i.e., holds for {\em any} $\bto$. 

We would also like to stress that while hard-thresholding based methods have been studied rigorously for the sparse-recovery problem \cite{BlumensathD09,JainTK14}, hard-thresholding has not been studied formally for the robust regression problem. Moreover, the two problems are completely different and hence techniques from sparse-recovery analysis do not extend to robust regression. 

Despite its simplicity, \alg-FC does not scale very well to datasets with large $p$ as it solves least squares problems at each iteration. We address this issue by designing a gradient descent based algorithm (\alg-GD), and a hybrid algorithm (\alg-Hyb), both of which enjoy a geometric rate of convergence and can recover $\bto$ under the model assumptions mentioned above. We also propose extensions of \alg for the RLSR problem in the sparse regression setting where $p \gg n$ but $\norm{\bto}_0 = s^\ast \ll p$. Our algorithm \alg-HD is based on \alg-FC but uses the Iterative Hard Thresholding (IHT) algorithm, a popular algorithm for sparse regression. As before, we show that \alg-HD also converges geometrically to $\bto$ if a) the corruption index $\alpha$ is less than some constant $C$, b) $X$ is sampled from a sub-Gaussian distribution and, c) $n\geq s^*\log p$. 

Finally, we experimentally evaluate existing $L_1$-based algorithms and our hard thresholding-based algorithms. The results demonstrate that our proposed algorithms (\alg-(FC/GD/HYB)) can be significantly faster than the best $L_1$ solvers, exhibit better recovery properties, as well as be more robust to dense white noise. For instance, on a problem with $50K$ dimensions and $40\%$ corruption, \alg-HYB was found to be $20\times$ faster than $L_1$ solvers, as well as achieve lower error rates.

{\bf Paper Organization}: We give a formal definition of the RLSR problem in the next section. We then introduce our family of algorithms in Section~\ref{sec:method} and prove their convergence guarantees in Section~\ref{sec:conv}. We present extensions to sparse robust regression in Section~\ref{sec:high-d} and empirical results in Section~\ref{sec:exps}. 

\section{Problem Formulation}
\label{sec:formulation}
Given a set of data points $X = [\x_1,\x_2,\ldots,\x_n]$, where $\x_i \in \R^p$ and the corresponding response vector $\y\in \R^n$, the goal is to recover a parameter vector $\bto$ which solves the RLSR problem \eqref{eq:rlsr}. We assume that the response vector $\y$ is generated using the following model: $$\y=\y^\ast+\b+\vve,\ \mbox{where}\ \y^\ast=X^\top\bto.$$
Hence, in the above model, \eqref{eq:rlsr} reduces to estimating $\bto$. We allow the model $\bto$ representing the regressor, to be chosen in an adaptive manner \emph{after} the data features have been generated. 

The above model allows two kinds of perturbations to $y_i$ -- dense but bounded noise $\ve_i$ (e.g. white noise $\ve_i \sim \cN(0,\sigma^2), \sigma \geq 0$), as well as potentially unbounded corruptions $b_i$ -- to be introduced by an adversary. The only requirement we enforce is that the gross corruptions be sparse.  

$\vve$ shall represent the dense noise vector, for example $\vve \sim \cN(\vz,\sigma^2\cdot I_{n\times n})$, and $\b$, the corruption vector such that $\norm{\b}_0 \leq \alpha\cdot n$ for some \emph{corruption index} $\alpha > 0$. We shall use the notation $S_\ast = \overline{\text{supp}(\b)} \subseteq [n]$ to denote the set of ``clean'' points, i.e. points that have not faced unbounded corruptions. We  consider adaptive adversaries that are able to view the generated data points $\x_i$, as well as the clean responses $y^\ast_i$ and dense noise values $\ve_i$ before deciding which locations to corrupt and by what amount.

We  denote the unit sphere in $p$ dimensions using $S^{p-1}$. For any $\gamma \in (0,1]$, we let $\cS_\gamma = \bc{S \subset [n]: |S| = \gamma\cdot n}$ denote the set of all subsets of size $\gamma\cdot n$. For any set $S$, we let $X_S := \bs{\x_i}_{i \in S} \in \R^{p \times \abs{S}}$ denote the matrix whose columns are composed of points in that set. Also, for any vector $\v \in \R^n$ we  use the notation $\v_S$ to denote the $\abs{S}$-dimensional vector consisting of those components that are in $S$. We  use $\lambda_{\min}(X)$ and $\lambda_{\max}(X)$ to denote, respectively, the smallest and largest eigenvalues of a square symmetric matrix $X$. We now introduce two properties, namely, \emph{Subset Strong Convexity} and \emph{Subset Strong Smoothness}, which are key to our analyses. 

\begin{defn}[SSC and SSS Properties]
\label{defn:ssc-sss}
A matrix $X \in \R^{p\times n}$ satisfies the \emph{Subset Strong Convexity Property} (resp. \emph{Subset Strong Smoothness Property}) at level $\gamma$ with strong convexity constant $\lambda_\gamma$ (resp. strong smoothness constant $\Lambda_\gamma$) if the following holds:
\[
\lambda_\gamma \leq \underset{S\in\cS_\gamma}{\min} \lambda_{\min}(X_SX_S^\top) \leq \underset{S\in\cS_\gamma}{\max} \lambda_{\max}(X_SX_S^\top) \leq \Lambda_\gamma.
\]
\end{defn}

\begin{remark}
We note that the uniformity enforced in the definitions of the SSC and SSS properties is not for the sake of convenience but rather a necessity. Indeed, a uniform bound is required in face of an adversary which can perform corruptions \emph{after} data and response variables have been generated, and choose to corrupt precisely that set of points where the SSC and SSS parameters are the worst.
\end{remark}


\section{\alg: Thresholding Operator-based Robust Regression Method}
\label{sec:method}

\begin{figure*}[t]
\begin{minipage}[t]{0.49\linewidth}
\begin{algorithm}[H]
	\caption{\small \alg: Thresholding Operator-based Robust RegrEssioN meThod}
	\label{algo:fcht}
	\begin{algorithmic}[1]
		\small{
			\REQUIRE Training data $\bc{\x_i,y_i}, i = 1 \ldots n$, step length $\eta$, thresholding parameter $\beta$, tolerance $\epsilon$
			\STATE $\bt^0 \< \vz, S_0 = [n], t \< 0, \r^0 \< \y$
			\WHILE{$\norm{\r^t_{S_{t}}}_2 > \epsilon$}
				\STATE $\btn \< \text{UPDATE}(\btt,S_t,\eta,\r^t,S_{t-1})$
				\STATE $r^{t+1}_i \< \br{y_i - \ip{\btn}{\x_i}}$
				\STATE $\displaystyle S_{t+1} \< \ght(\r^{t+1},(1-\beta) n)$
				\STATE $t \< t + 1$
			\ENDWHILE
			\STATE \textbf{return} $\btt$
		}
	\end{algorithmic}
\end{algorithm}
\vskip-3ex
\begin{algorithm}[H]
	\caption{\small \alg-FC}
	\label{algo:update-fc}
	\begin{algorithmic}[1]
		\small{
			\REQUIRE Current model $\bt$, current active set $S$
			\STATE \textbf{return} $\displaystyle \underset{\bt}{\arg\min}\sum_{i \in S}\br{y_i - \ip{\bt}{\x_i}}^2$
		}
	\end{algorithmic}
\end{algorithm}
\end{minipage}
\begin{minipage}[t]{0.5\linewidth}
\begin{algorithm}[H]
	\caption{\small \alg-GD}
	\label{algo:update-grades}
	\begin{algorithmic}[1]
		\small{
			\REQUIRE Current model $\bt$, current active set $S$, step size $\eta$
			\STATE \vskip-2.5ex$\g \< X_S(X_S^\top\bt - \y_S)$
			\STATE \textbf{return} $\bt - \eta\cdot \g$
		}
	\end{algorithmic}
\end{algorithm}
\vskip-2.5ex
\begin{algorithm}[H]
	\caption{\small \alg-HYB}
	\label{algo:update-hybrid}
	\begin{algorithmic}[1]
		\small{
			\REQUIRE Current model $\bt$, current active set $S$, step size $\eta$, current residuals $\r$, previous active set $S'$
			\STATE \COMMENT{ Use the GD update if the active set $S$ is changing a lot}
			\IF{$\abs{S\backslash S'} > \Delta$}
				\STATE $\bt' \< \text{UPDATE-GD}(\bt,S,\eta,\r,S')$
			\ELSE
				\STATE \hskip-2.5ex\COMMENT{ If stable, use the FC update}
				\STATE $\bt' \< \text{UPDATE-FC}(\bt,S)$
			\ENDIF
			\STATE \textbf{return} $\bt'$
		}
	\end{algorithmic}
\end{algorithm}
\end{minipage}
\end{figure*}

We now present \alg, a Thresholding Operator-based Robust RegrEssioN meThod for performing robust regression at scale. Key to our algorithms is the \emph{Hard Thresholding Operator} which we define below.
\begin{defn}[Hard Thresholding Operator]
For any vector $\v \in \R^n$, let $\sigma_\v \in S_n$ be the permutation that orders elements of $\v$ in ascending order of their magnitudes i.e. $\abs{\v_{\sigma_\v(1)}} \leq \abs{\v_{\sigma_\v(2)}} \leq \ldots \leq \abs{\v_{\sigma_\v(n)}}$. Then for any $k \leq n$, we define the hard thresholding operator as
\[
\ght(\v;k) = \bc{i \in [n]: \sigma_\v^{-1}(i) \leq k}
\]
\end{defn}

Using this operator, we present our algorithm \alg (Algorithm~\ref{algo:fcht}) for robust regression. \alg follows a most natural iterative strategy of, alternately, estimating an \emph{active set} of points which have the least residual error on the current regressor, and then updating the regressor to provide a better fit on this active set. We offer three variants of our algorithm, based on how aggressively the algorithm tries to fit the regressor to the current active set.

We first propose a fully corrective algorithm \alg-FC (Algorithm~\ref{algo:update-fc}) that performs a fully corrective least squares regression step in an effort to minimize the regression error on the active set. This algorithm makes significant progress in each step, but at a cost of more expensive updates. To address this, we then propose a milder, gradient descent-based variant \alg-GD (Algorithm~\ref{algo:update-grades}) that performs a much cheaper update of taking a single step in the direction of the gradient of the objective function on the active set. This reduces the regression error on the active set but does not minimize it. This turns out to be beneficial in situations where dense noise is present along with sparse corruptions since it prevents the algorithm from overfitting to the current active set.

Both the algorithms proposed above have their pros and cons -- the FC algorithm provides significant improvements with each step, but is expensive to execute whereas the GD variant, although efficient in executing each step, offers slower progress. To get the best of both these algorithms, we propose a third, hybrid variant \alg-HYB (Algorithm~\ref{algo:update-hybrid}) that adaptively selects either the FC or the GD update depending on whether the active set is stable across iterations or not.

In the next section we show that this hard thresholding-based strategy offers a linear convergence rate for the algorithm in all its three variations. We shall also demonstrate the applicability of this technique to high dimensional sparse recovery settings in a subsequent section.


\section{Convergence Guarantees}
\label{sec:conv}

For the sake of ease of exposition, we will first present our convergence analyses for cases where dense noise is not present i.e. $\y = X^\top\bto + \b$ and will handle cases with dense noise \emph{and} sparse corruptions later. We first analyze the fully corrective \alg-FC algorithm. The convergence proof in this case relies on the optimality of the two steps carried out by the algorithm, the fully corrective step that selects the best regressor on the active set, and the hard thresholding step that discovers a new active set by selecting points with the least residual error on the current regressor.

\begin{thm}
\label{thm:fcht}
Let $X = \bs{\x_1, \dots, \x_n}\in \R^{p\times n}$ be the given data matrix and $\y=X^T\bto+\b$ be the corrupted output with $\|\b\|_0\leq \alpha\cdot n$. Let Algorithm~\ref{algo:update-fc} be executed on this data with the thresholding parameter set to $\beta \geq \alpha$. Let $\Sigma_0$ be an invertible matrix such that $\Xt=\isSo X$ satisfies the SSC and SSS properties at level $\gamma$ with constants $\lambda_{\gamma}$ and $\Lambda_{\gamma}$ respectively (see Definition~\ref{defn:ssc-sss}). If the data satisfies $\frac{(1+\sqrt 2)\Lambda_{\beta}}{\lambda_{1-\beta}} < 1$, then after $t = \O{\log\br{\frac{1}{\sqrt n}\frac{\norm{\b}_2}{\epsilon}}}$ iterations, Algorithm~\ref{algo:update-fc} obtains an $\epsilon$-accurate solution $\btt$ i.e. $\norm{\btt - \bto}_2 \leq \epsilon$.
\end{thm}
\begin{proof}[Proof (Sketch)]
Let $\r^t = \y - X^\top\btt$ be the vector of residuals at time $t$ and $C_t = X_{S_t}X_{S_t}^\top$. Also let $S_\ast = \overline{\text{supp}(\b)}$ be the set of uncorrupted points. The fully corrective step ensures that
\[
\btn = C_t^{-1}X_{S_t}\y_{S_t} = C_t^{-1}X_{S_t}\br{X_{S_t}^\top\bto + \b_{S_t}} = \bto + C_t^{-1}X_{S_t}\b_{S_t},
\]
whereas the hard thresholding step ensures that $\norm{\r^{t+1}_{S_{t+1}}}_2^2 \leq \norm{\r^{t+1}_{S_\ast}}_2^2$. Combining the two gives us
\begin{align*}
\norm{\b_{S_{t+1}}}_2^2 &\leq \norm{X_{S_\ast\backslash S_{t+1}}^\top C_t^{-1}X_{S_t}\b_{S_t}}_2^2 + 2\cdot\b_{S_{t+1}}^\top X_{S_{t+1}}^\top C_t^{-1}X_{S_t}\b_{S_t}\\
&\stackrel{\zeta_1}{=}\norm{\Xt_{S_\ast\backslash S_{t+1}}^\top \left(\Xt_{S_t}\Xt_{S_t}^T\right)^{-1}\Xt_{S_t}\b_{S_t}}_2^2 + 2\cdot\b_{S_{t+1}}^\top \Xt_{S_{t+1}}^\top \left( \Xt_{S_t}\Xt_{S_t}^T\right)^{-1} \Xt_{S_t}\b_{S_t}\\
&\stackrel{\zeta_2}{\leq} {\frac{\Lambda_{\beta}^2}{\lambda_{1-\beta}^2}}\cdot\norm{\b_{S_t}}_2^2 + 2\cdot\frac{\Lambda_{\beta}}{\lambda_{1-\beta}}\cdot\norm{\b_{S_t}}_2\norm{\b_{S_{t+1}}}_2,
\end{align*}
where $\zeta_1$ follows from setting $\Xt=\isSo X$ and $X_S^\top C_t^{-1}X_{S'}=\Xt_S^\top(\Xt_{S_t}\Xt_{S_t}^\top)^{-1}\Xt_{S'}$ and $\zeta_2$ follows from the SSC and SSS properties, $\norm{\b_{S_t}}_0 \leq \norm{\b}_0 \leq \beta\cdot n$ and $\abs{S_\ast\backslash S_{t+1}} \leq \beta\cdot n$. Solving the quadratic equation and performing other manipulations gives us the claimed result.
\end{proof}

Theorem~\ref{thm:fcht} relies on a deterministic (\emph{fixed design}) assumption, specifically $\frac{(1+\sqrt 2)\Lambda_{\beta}}{\lambda_{1-\beta}} < 1$ in order to guarantee convergence. We can show that a large class of random designs, including Gaussian and sub-Gaussian designs actually satisfy this requirement. That is to say, data generated from these distributions satisfy the SSC and SSS conditions such that $\frac{(1+\sqrt 2)\Lambda_{\beta}}{\lambda_{1-\beta}} < 1$ with high probability. Theorem~\ref{thm:fcht-explicit-rate} explicates this for the class of Gaussian designs.
\vskip2ex
\begin{thm}
\label{thm:fcht-explicit-rate}
Let $X=[\x_1, \dots, \x_n]\in \R^{p\times n}$ be the given data matrix with each $\x_i \sim \cN(\vz	, \Sigma)$. Let $\y=X^\top\bto+\b$ and $\|\b\|_0\leq \alpha \cdot n$. Also, let $\alpha \leq \beta < \frac{1}{65}$ and $n \geq \Om{p + \log\frac{1}{\delta}}$. Then, with probability at least $1-\delta$, the data satisfies $\frac{(1+\sqrt 2)\Lambda_{\beta}}{\lambda_{1-\beta}} < \frac{9}{10}$. More specifically, after $T \geq 10\log\br{\frac{1}{\sqrt n}\frac{\norm{\b}_2}{\epsilon}}$ iterations of Algorithm~\ref{algo:fcht} with the thresholding parameter set to $\beta$, we have $\norm{\bt^T-\bto}\leq \epsilon$.
\end{thm}
\vskip1ex
\begin{remark}
Note that Theorem~\ref{thm:fcht-explicit-rate} provides rates that are independent of the condition number $\frac{\lambda_{\max}(\Sigma)}{\lambda_{\min}(\Sigma)}$ of the distribution.  We also note that results similar to Theorem~\ref{thm:fcht-explicit-rate} can be proven for the larger class of sub-Gaussian distributions. We refer the reader to Section~\ref{sec:stat} for the same.
\end{remark}
\vskip1ex
\begin{remark}
We remind the reader that our analyses can readily accommodate dense noise in addition to sparse unbounded corruptions. We direct the reader to Appendix~\ref{app:dense-noise} which presents convergence proofs for our algorithms in these settings.
\end{remark}
\vskip1ex
\begin{remark}
We would like to point out that the design requirements made by our analyses are very mild when compared to existing literature. Indeed, the work of \cite{WrightM10} assumes the \emph{Bouquet Model} where distributions are restricted to be isotropic Gaussians whereas the work of \cite{NguyenT13} assumes a more stringent model of sub-orthonormal matrices, something that even Gaussian designs do not satisfy. Our analyses, on the other hand, hold for the general class of sub-Gaussian distributions.
\end{remark}

We now analyze the \alg-GD algorithm which performs cheaper, gradient-style updates on the active set. We will show that this method nevertheless enjoys a linear rate of convergence.

\begin{thm}
\label{thm:fcht-grades}
Let the data settings be as stated in Theorem~\ref{thm:fcht} and let Algorithm~\ref{algo:update-grades} be executed on this data with the thresholding parameter set to $\beta \geq \alpha$ and the step length set to $\eta = \frac{1}{\Lambda_{1-\beta}}$.  If the data satisfies $\max\bc{\eta\sqrt{\Lambda_\beta}, 1 - \eta\lambda_{1-\beta}} \leq \frac{1}{4}$, then after $t = \O{\log\br{\frac{\norm{b}_2}{\sqrt{n}}\frac{1}{\epsilon}}}$ iterations, Algorithm~\ref{algo:fcht} obtains an $\epsilon$-accurate solution $\btt$ i.e. $\norm{\btt - \bto}_2 \leq \epsilon$.
\end{thm}

Similar to \alg-FC, the assumptions made by the \alg-GD algorithm are also satisfied by the class of sub-Gaussian distributions. The proof of Theorem~\ref{thm:fcht-grades}, given in Appendix~\ref{app:thm-fcht-grades}, details these arguments. Given the convergence analyses for \alg-FC and GD, we now move on to provide a convergence analysis for the hybrid \alg-HYB algorithm which interleaves FC and GD steps. Since the exact interleaving adopted by the algorithm depends on the data, and not known in advance, this poses a problem. We address this problem by giving below a uniform convergence guarantee, one that applies to \emph{every interleaving} of the FC and GD update steps.

\begin{thm}
\label{thm:fcht-hyb-rate}
Suppose Algorithm~\ref{algo:update-hybrid} is executed on data that allows Algorithms~\ref{algo:update-fc}~and~\ref{algo:update-grades} a convergence rate of $\eta_\fc$ and $\eta_\gd$ respectively. Suppose we have $2\cdot\eta_\fc\cdot\eta_\gd < 1$. Then for \emph{any} interleavings of the FC and GD steps that the policy may enforce, after $t = \O{\log\br{\frac{1}{\sqrt n}\frac{\norm{\b}_2}{\epsilon}}}$ iterations, Algorithm~\ref{algo:update-hybrid} ensures an $\epsilon$-optimal solution i.e. $\norm{\btt - \bto} \leq \epsilon$.
\end{thm}

We point out to the reader that the assumption made by Theorem~\ref{thm:fcht-hyb-rate} i.e. $2\cdot\eta_\fc\cdot\eta_\gd < 1$ is readily satisfied by random sub-Gaussian designs, albeit at the cost of reducing the noise tolerance limit. As we shall see, \alg-HYB offers attractive convergence properties, merging the fast convergence rates of the FC step, as well as the speed and protection against overfitting provided by the GD step.


\section{High-dimensional Robust Regression}
\label{sec:high-d}
In this section, we extend our approach to the robust high-dimensional sparse recovery setting. As before, we assume that the response vector $\y$ is obtained as: 
$\y=X^\top\bto+\b$, where $\|\b\|_0\leq \alpha \cdot n$. However, this time, we also assume that $\bto$ is $s^*$-sparse i.e. $\norm{\bto}_0 \leq s^\ast$. 

As before, we shall neglect white/dense noise for the sake of simplicity. We reiterate that it is not possible to use existing results from sparse recovery (such as \cite{BlumensathD09, JainTK14}) directly to solve this problem.

Our objective would be to recover a \emph{sparse} model $\hat\bt$ so that $\norm{\hat\bt - \bto}_2 \leq \epsilon$. The challenge here is to forgo a sample complexity of $n \gtrsim p$ and instead, perform recovery with $n \sim s^\ast\log p$ samples alone. For this setting, we modify the FC update step of \alg-FC method to the following: 
\begin{align}
\label{eq:sparse-recov}
\btn \< \underset{\norm{\bt}_0 \leq s}{\inf}\sum_{i \in S_t}\br{y_i - \ip{\bt}{\x_i}}^2,
\end{align}
for some \emph{target} sparsity level $s \ll p$. We refer to this modified algorithm as \alg-HD. Assuming $X$ satisfies the RSC/RSS properties (defined below), \eqref{eq:sparse-recov} can be solved efficiently using results from sparse recovery (for example the IHT algorithm \cite{BlumensathD09,GargK09}  analyzed in \cite{JainTK14}).

\begin{defn}[RSC and RSS Properties]
\label{defn:rsc-rss}
A matrix $X \in \R^{p\times n}$ will be said to satisfy the \emph{Restricted Strong Convexity Property} (resp. \emph{Restricted Strong Smoothness Property}) at level $s = s_1+s_2$ with strong convexity constant $\alpha_{s_1+s_2}$ (resp. strong smoothness constant $L_{s_1+s_2}$) if the following holds for all $\norm{\bt_1}_0 \leq s_1$ and $\norm{\bt_2}_0 \leq s_2$:
\[
\alpha_{s} \norm{\bt_1 - \bt_2}_2^2 \leq \norm{X^\top (\bt_1 - \bt_2)}_2^2 \leq L_{s} \norm{\bt_1 - \bt_2}_2^2
\]
\end{defn}

For our results, we shall require the subset versions of both these properties.

\begin{defn}[SRSC and SRSS Properties]
\label{defn:srsc-srss}
A matrix $X \in \R^{p\times n}$ will be said to satisfy the \emph{Subset Restricted Strong Convexity} (resp. \emph{Subset Restricted Strong Smoothness}) Property at level $(\gamma,s)$ with strong convexity constant $\alpha_{(\gamma,s)}$ (resp. strong smoothness constant $L_{(\gamma,s)}$) if for all subsets $S \in \cS_\gamma$, the matrix $X_S$ satisfies the RSC (resp. RSS) property at level $s$ with constant $\alpha_s$ (resp. $L_s$).
\end{defn}

We now state the convergence result for the \alg-HD algorithm.

\begin{thm}
\label{thm:fcht-highd}
Let $X\in \R^{p\times n}$ be the given data matrix and $\y=X^T\bto+\b$ be the corrupted output with $\norm{\bto}_0 \leq s^\ast$ and $\|\b\|_0\leq \alpha\cdot n$. Let $\Sigma_0$ be an invertible matrix such that $\isSo X$ satisfies the SRSC and SRSS properties at level $(\gamma,2s+s^\ast)$ with constants $\alpha_{(\gamma,2s+s^\ast)}$ and $L_{(\gamma,2s+s^\ast)}$ respectively (see Definition~\ref{defn:srsc-srss}). Let Algorithm~\ref{algo:update-fc} be executed on this data with the \alg-\textup{HD} update, thresholding parameter set to $\beta \geq \alpha$, and $s \geq 32\br{\frac{L_{(1-\beta,2s+s^\ast)}}{\alpha_{(1-\beta,2s+s^\ast)}}}$. 

If $X$ also satisfies $\frac{4L_{(\beta,s+s^\ast)}}{\alpha_{(1-\beta,s+s^\ast)}} < 1$, then after $t = \O{\log\br{\frac{1}{\sqrt n}\frac{\norm{\b}_2}{\epsilon}}}$ iterations, Algorithm~\ref{algo:update-fc} obtains an $\epsilon$-accurate solution $\btt$ i.e. $\norm{\btt - \bto}_2 \leq \epsilon$. 

In particular, if $X$ is sampled from a Gaussian distribution $\cN(\vz,\Sigma)$ and $n \geq \Om{s^\ast\cdot \frac{\lambda_{\max}(\Sigma)}{\lambda_{\min}(\Sigma)}\log p}$, then for all values of $\alpha \leq \beta < \frac{1}{65}$, we can guarantee $\norm{\btt - \bto}_2 \leq \epsilon$ after $t = \O{\log\br{\frac{1}{\sqrt n}\frac{\norm{\b}_2}{\epsilon}}}$ iterations of the algorithm (w.p. $\geq 1-1/n^{10}$). 
\end{thm}
\vskip1ex
\begin{remark}
The sample complexity required by Theorem~\ref{thm:fcht-highd} is identical to the one required by analyses for high dimensional sparse recovery \cite{JainTK14}, save constants. Also note that \alg-HD can tolerate the same corruption index as \alg-FC.
\end{remark}


\section{Experiments}
\label{sec:exps}

\begin{figure*}[t!]
	\centering
	\subfigure[\hspace*{-3ex}]{
		\includegraphics[scale=0.175]{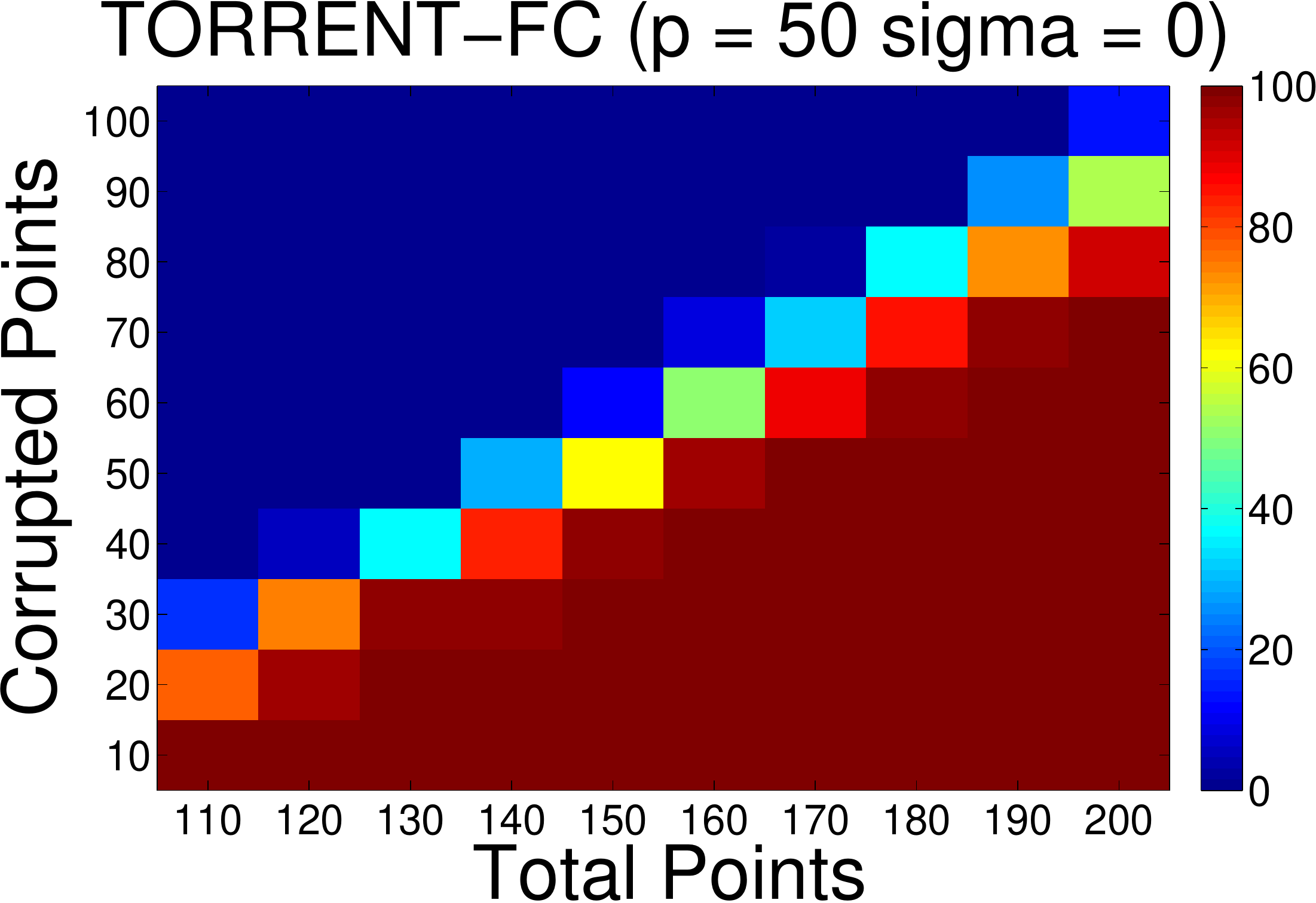}
	}\hspace*{-1ex}
	\subfigure[\hspace*{-3ex}]{
		\includegraphics[scale=0.175]{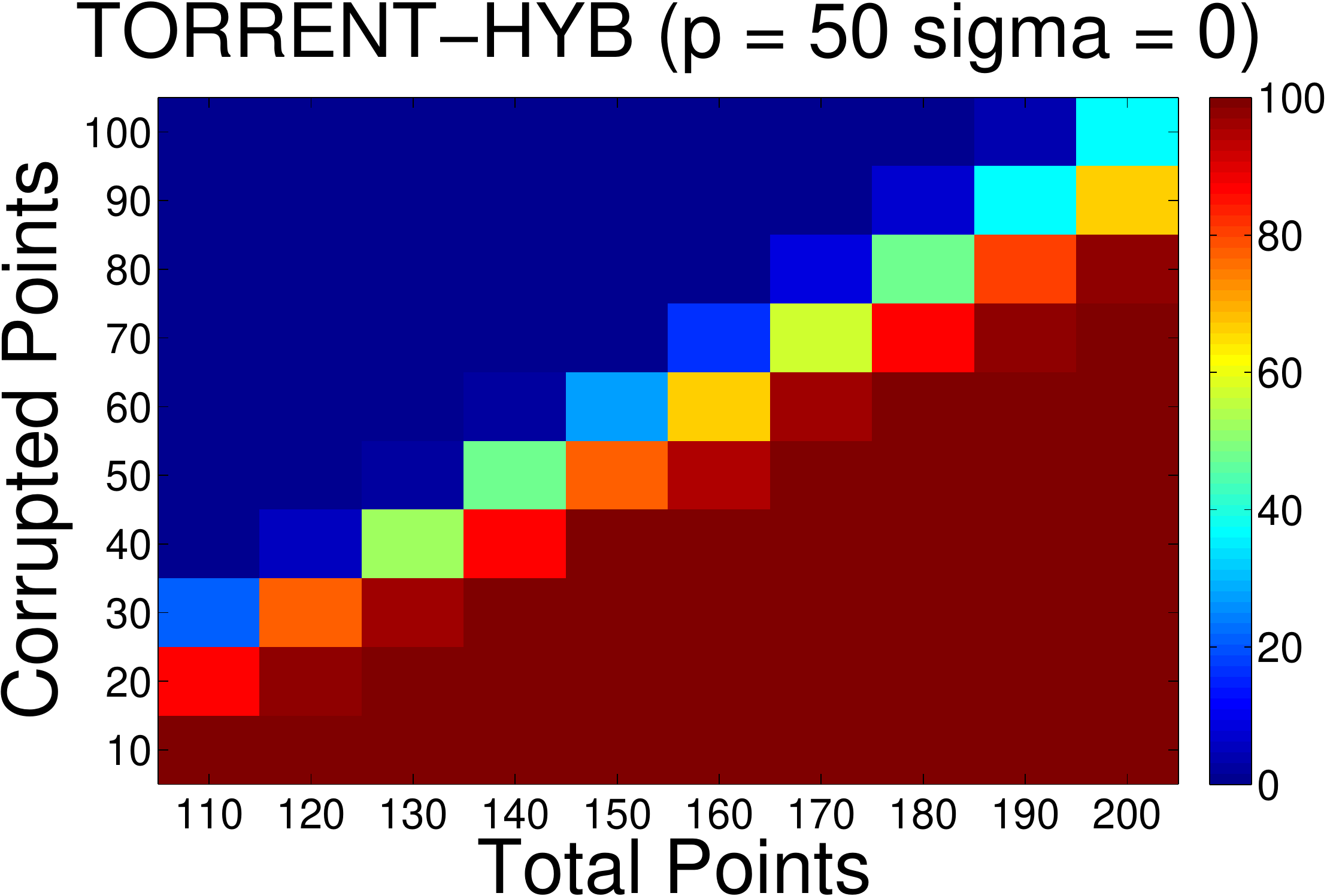}
	}\hspace*{-1ex}
	\subfigure[\hspace*{-3ex}]{
		\includegraphics[scale=0.175]{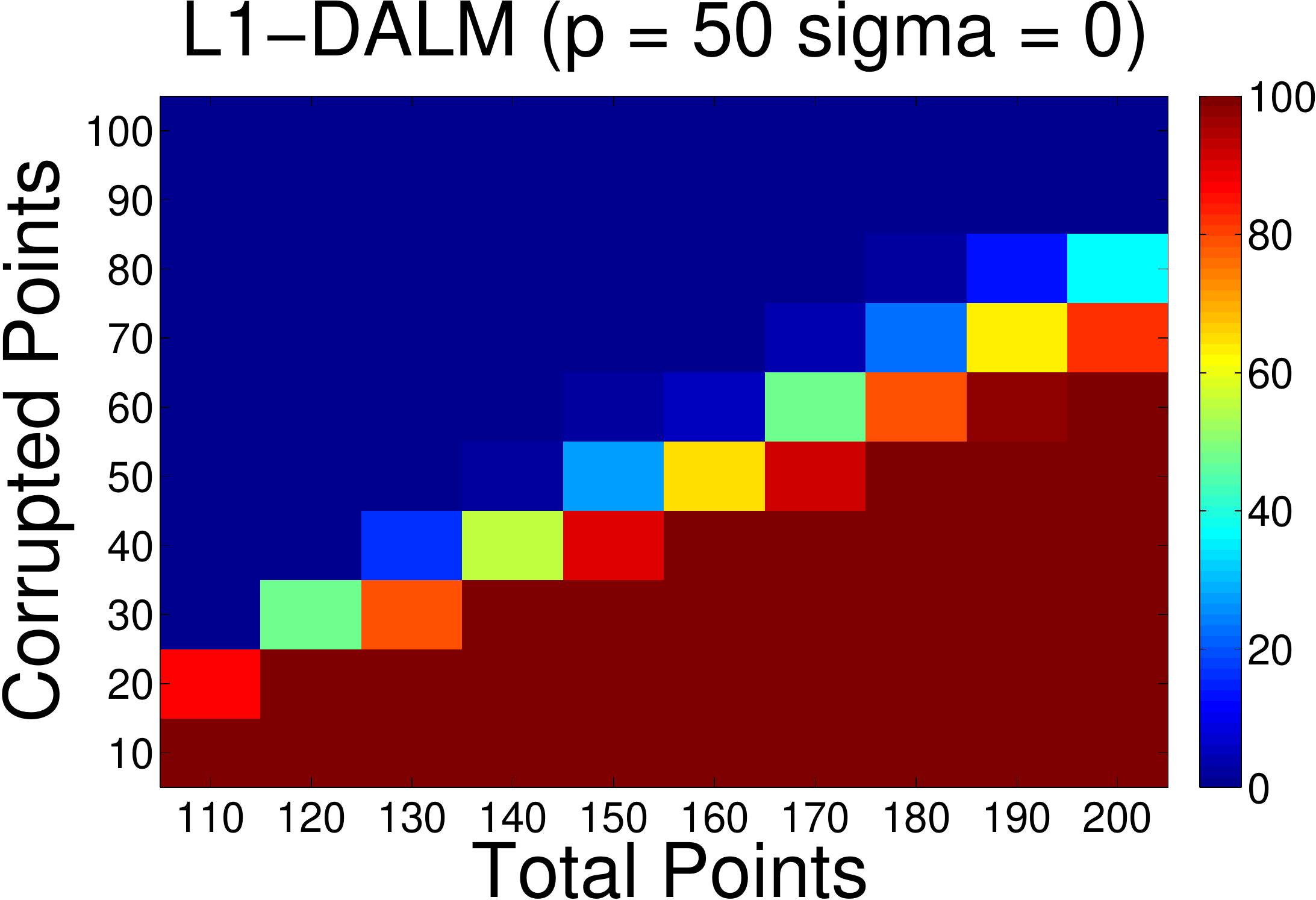}
	}\hspace*{-1ex}
	\subfigure[\hspace*{-3ex}]{
		\includegraphics[scale=0.175]{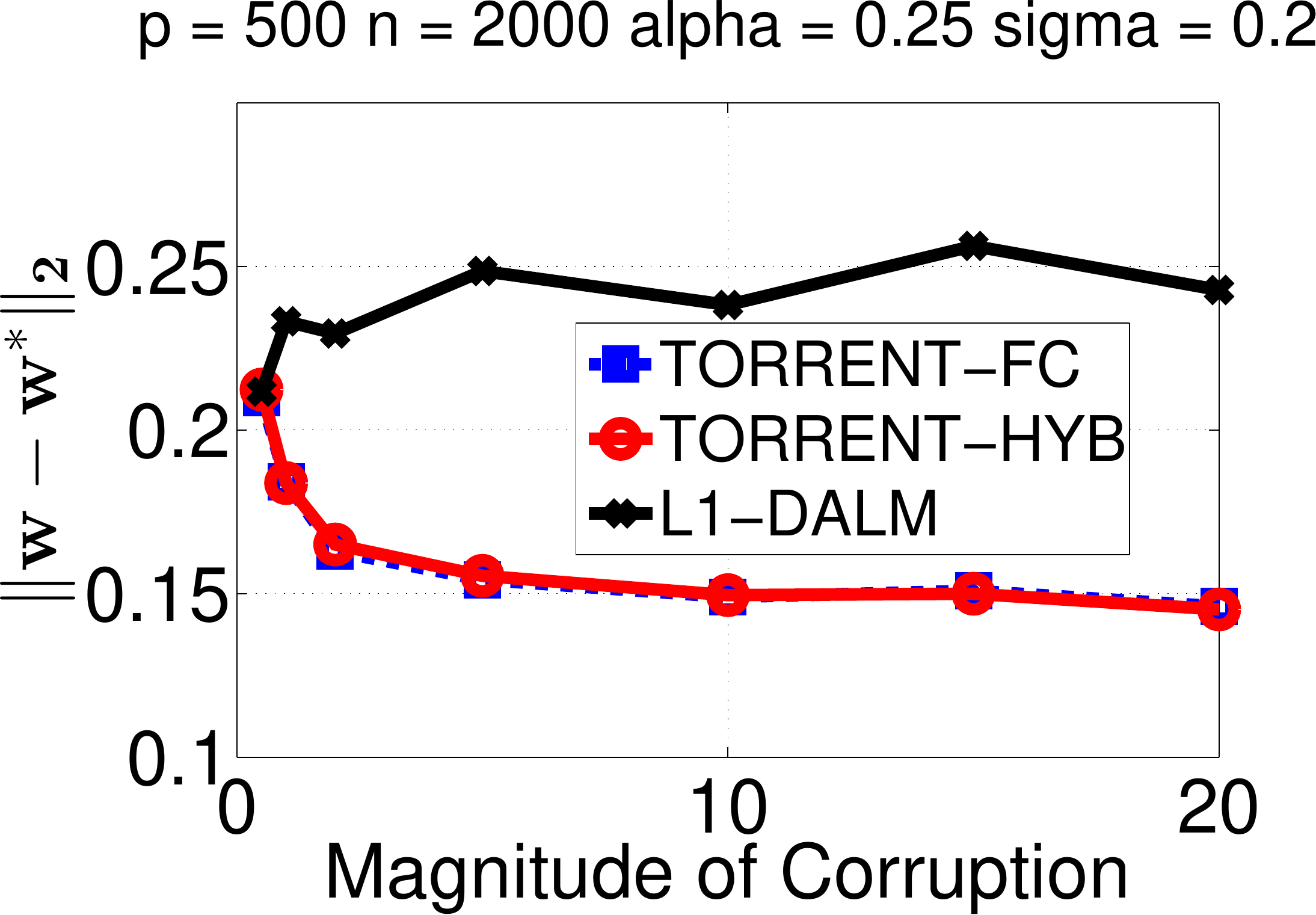}
	}
\hspace{2ex}
\caption{\small{(a), (b) and (c) Phase-transition diagrams depicting the recovery properties of the \alg-FC, \alg-HYB and $L_1$ algorithms. The colors red and blue represent a high and low probability of success resp. A method is considered successful in an experiment if it recovers $\bto$ upto a $10^{-4}$ relative error. Both variants of \alg can be seen to recover $\bto$ in presence of larger number of corruptions than the $L_1$ solver. (d) Variation in recovery error with the magnitude of corruption. As the corruption is increased, \alg-FC and \alg-HYB show improved performance while the problem becomes more difficult for the $L_1$ solver.}}
\label{fig:hm}
\end{figure*}

Several numerical simulations were carried out on linear regression problems in low-dimensional, as well as sparse high-dimensional settings. The experiments show that \alg not only offers statistically better recovery properties as compared to $L_1$-style approaches, but that it can be more than an order of magnitude faster as well.

\textbf{Data}: For the low dimensional setting, the regressor $\bto\in \R^p$ was chosen to be a random unit norm vector. Data was sampled as $\mathbf{x_i} \sim \cN(0,I_p)$ and response variables were generated as $y_i^\ast = \ip{\bto}{\x_i}$. The set of corrupted points $\overline{S}_\ast$ was selected as a uniformly random $(\alpha n)$-sized subset of $[n]$ and the corruptions were set to $b_i \sim U\br{-5 \norm{\mathbf{\y^\ast}}_{\infty}, 5 \norm{\mathbf{\y^\ast}}_{\infty}}$ for $i \in \overline{S}_\ast$. The corrupted responses were then generated as $y_i = y_i^\ast + b_i + \ve_i$ where $\ve_i\sim\cN(0, \sigma^2)$. For the sparse high-dimensional setting, $\text{supp}(\bto)$ was selected to be a random $s^*$-sized subset of $[p]$. Phase-transition diagrams (Figure \ref{fig:hm}) were generated by repeating each experiment 100 times. For all other plots, each experiment was run over 20 random instances of the data and the plots were drawn to depict the mean results.

\textbf{Algorithms}: We compared various variants of our algorithm \alg to the regularized $L_1$ algorithm for robust regression \cite{WrightM10,NguyenT13}. Note that the $L_1$ problem can be written as $\min_\z \norm{\z}_1 \text{s.t.} A\z=\y$, where $A= \left[X^\top\ \frac{1}{\lambda}I_{m\times m}\right]$ and  $\mathbf{z^*} = [\mathbf{w^{*\top}}\ \lambda \mathbf{b^\top}]^\top$. We used the Dual Augmented Lagrange Multiplier (DALM) $L_1$ solver implemented by ~\cite{YangGZSM12} to solve the $L_1$ problem. We ran a fine tuned grid search over the $\lambda$ parameter for the $L_1$ solver and quoted the best results obtained from the search. In the low-dimensional setting, we compared the recovery properties of \alg-FC (Algorithm~\ref{algo:update-fc}) and \alg-HYB (Algorithm~\ref{algo:update-hybrid}) with the DALM-$L_1$ solver, while for the high-dimensional case, we compared \alg-HD against the DALM-$L_1$ solver. Both the $L_1$ solver, as well as our methods, were implemented in Matlab and were run on a single core $2.4$GHz machine with $8$ GB RAM. 

\textbf{Choice of $L_1$-solver}: An extensive comparative study of various $L_1$ minimization algorithms was performed by \cite{YangGZSM12} who showed that the DALM and Homotopy solvers outperform other counterparts both in terms of recovery properties, and timings. We extended their study to our observation model and found the DALM solver to be significantly better than the other $L_1$ solvers; see Figure~\ref{fig:app} in the appendix. We also observed, similar to \cite{YangGZSM12}, that the Approximate Message Passing (AMP) solver diverges on our problem as the input matrix to the $L_1$ solver is a non-Gaussian matrix $A=[X^T \frac{1}{\lambda}I]$.

\textbf{Evaluation Metric}: We measure the performance of various algorithms using the standard $L_2$ error:  $r_{\widehat\bt} = \norm{\widehat\bt - \bto}_2$. For the phase-transition plots (Figure \ref{fig:hm}), we deemed an algorithm successful on an instance if it obtained a model $\widehat\bt$ with error $r_{\widehat\bt} < 10^{-4}\cdot\norm{\bto}_2$. We also measured the CPU time required by each of the methods, so as to compare their scalability.

\subsection{Low Dimensional Results}

\quad\textbf{Recovery Property}: The phase-transition plots presented in Figure \ref{fig:hm} represent our recovery experiments in graphical form. Both the fully-corrective and hybrid variants of \alg show better recovery properties than the $L_1$-minimization approach, indicated by the number of runs in which the algorithm was able to correctly recover $\bto$ out of a 100 runs. Figure \ref{fig:recoveryPlots} shows the variation in recovery error as a function of $\alpha$ in the presence of white noise and exhibits the superiority of \alg-FC and \alg-HYB over $L_1$-DALM. Here again, \alg-FC and \alg-HYB achieve significantly lesser recovery error than $L_1$-DALM for all $\alpha <= 0.5$. Figure \ref{fig:app} in the appendix show that the variations of $\norm{\widehat\bt - \bto}_2$ with varying $p, \sigma$ and $n$ follow a similar trend with \alg having significantly lower recovery error in comparison to the $L_1$ approach. 

Figure ~\ref{fig:hm}(d) brings out an interesting trend in the recovery property of \alg. As we increase the magnitude of corruption from $U\br{-\norm{\y^\ast}_\infty, \norm{\y^\ast}_\infty}$ to $U\br{-20\norm{\y^\ast}_\infty, 20\norm{\y^\ast}_\infty}$, the recovery error for \alg-HYB and \alg-FC decreases as expected since it becomes easier to identify the grossly corrupted points. However the $L_1$-solver was unable to exploit this observation and in fact exhibited an increase in recovery error.
\begin{figure}[t!]
	\centering
	\hspace*{-1ex}
	\subfigure[\hspace*{-3ex}]{
		\includegraphics[scale=0.175]{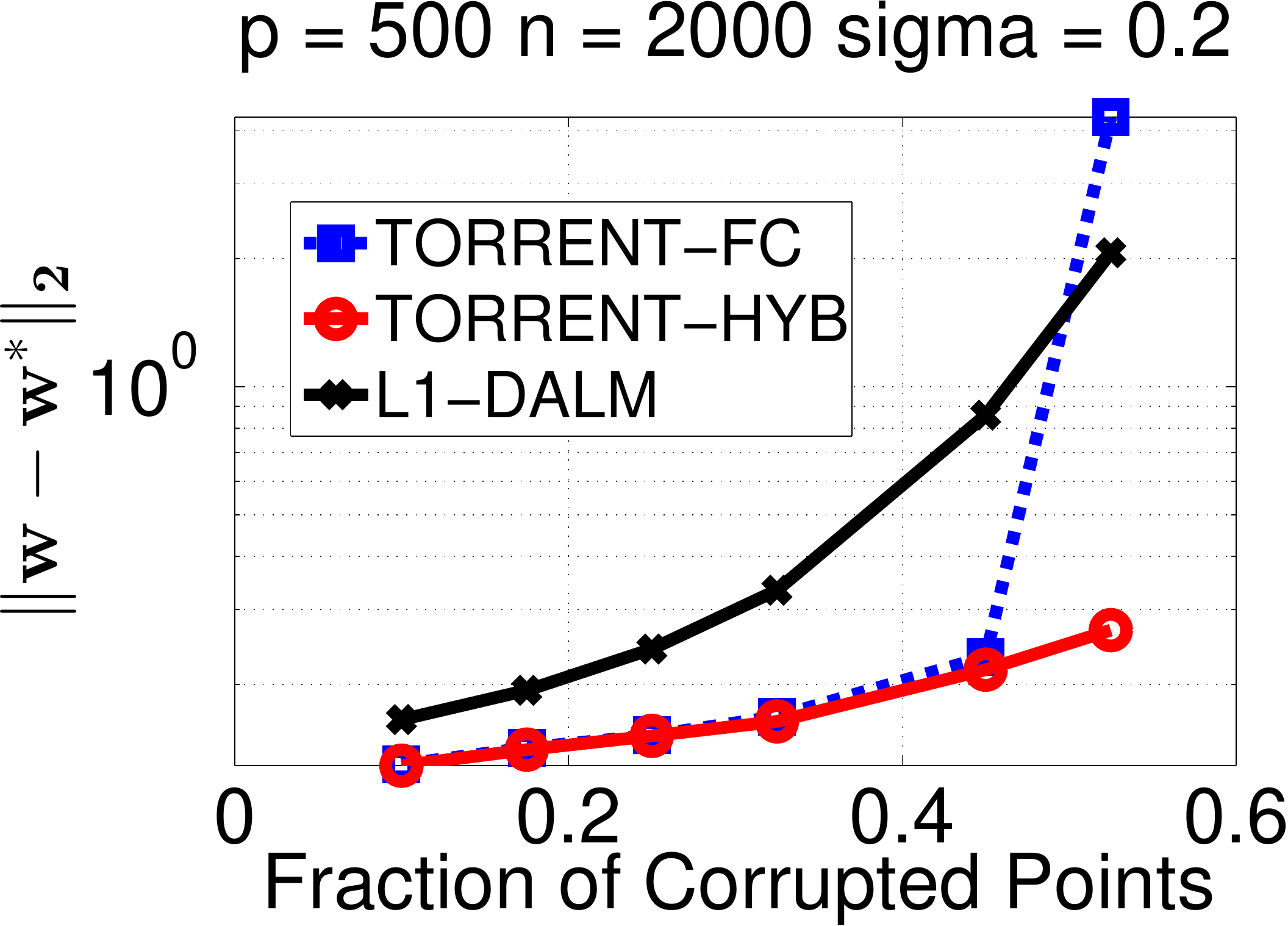}
	}\hspace*{-1ex}
	\subfigure[\hspace*{-3ex}]{
		\includegraphics[scale=0.1875]{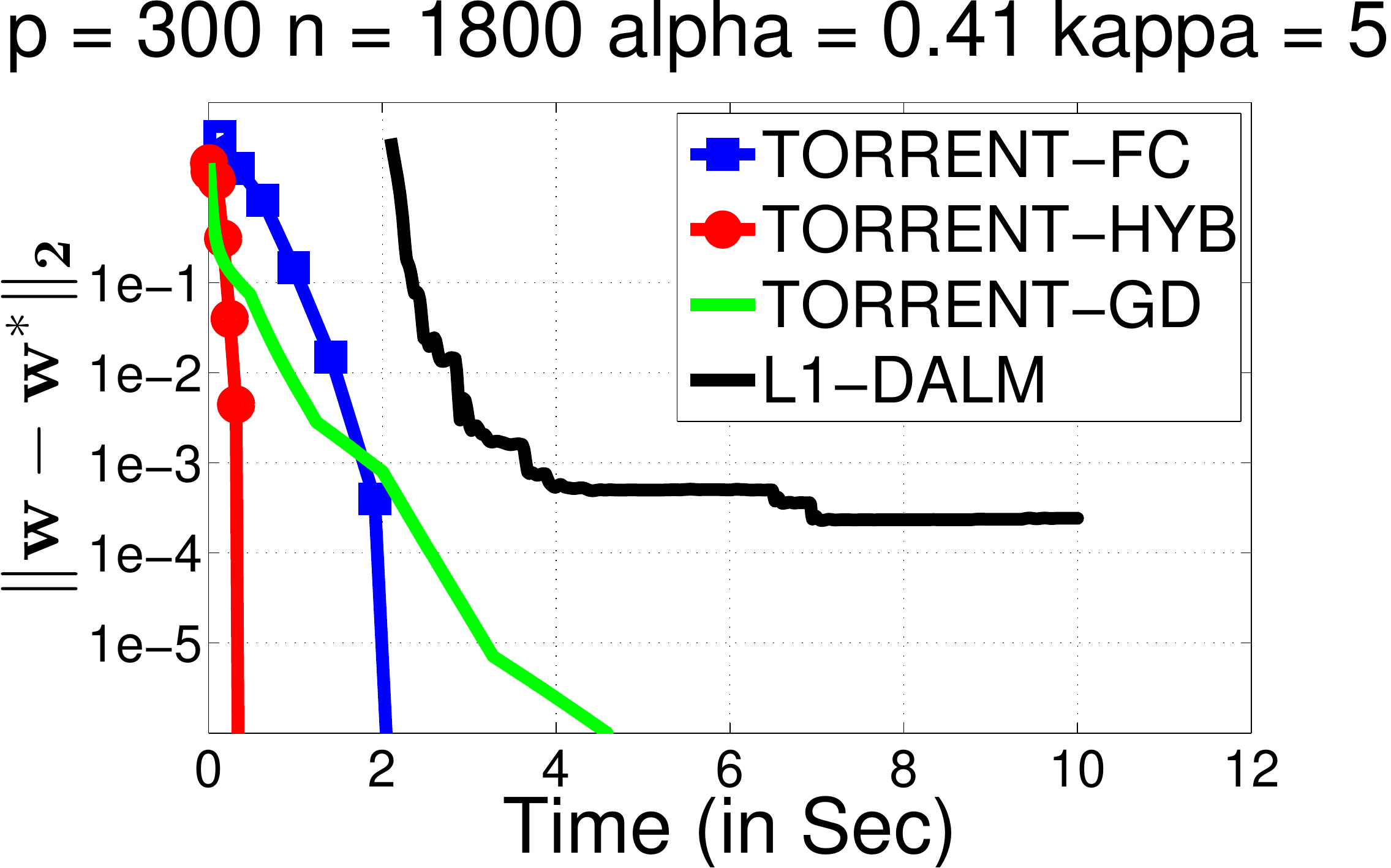}
	}\hspace*{-1ex}
	\subfigure[\hspace*{-3ex}]{
		\includegraphics[scale=0.175]{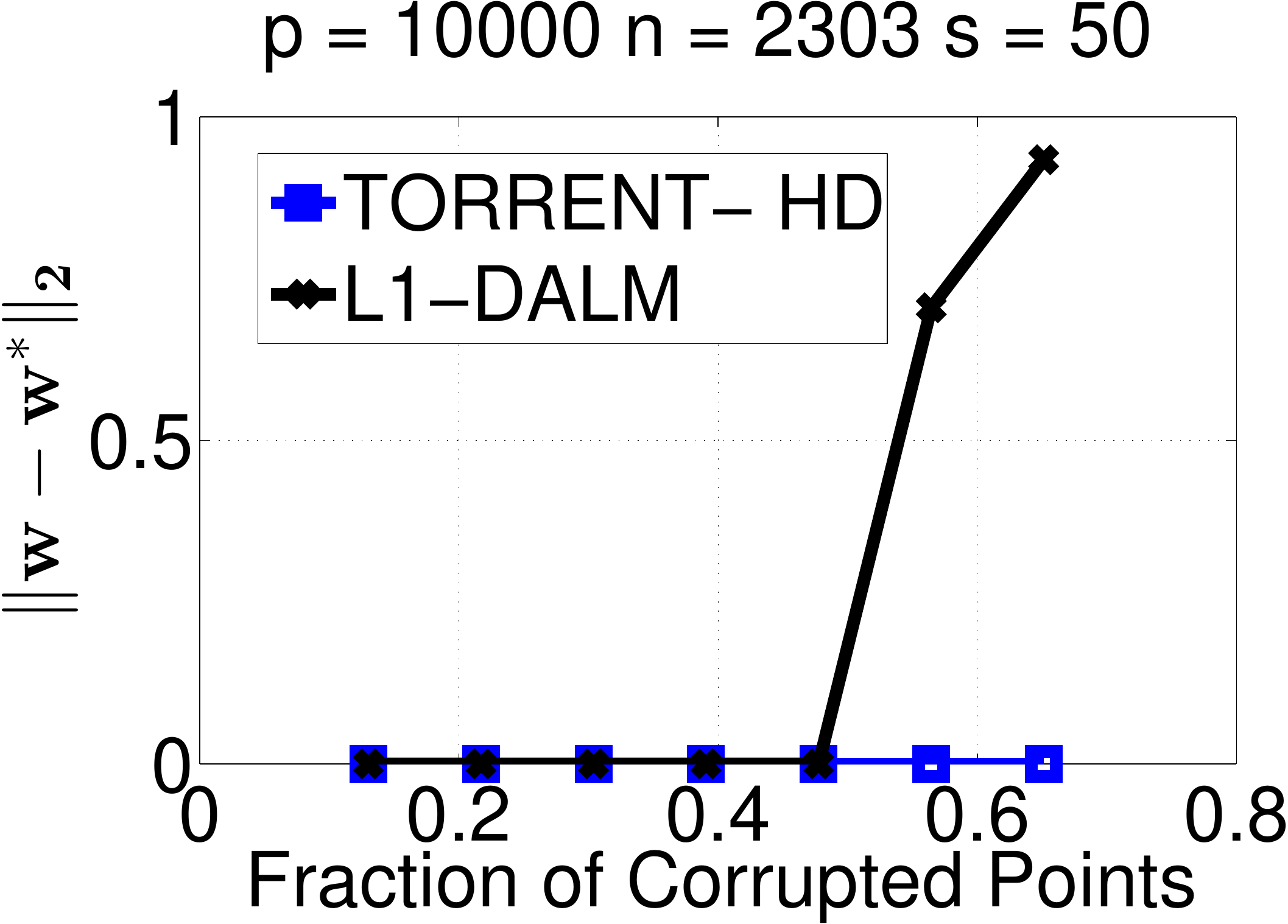}
	}\hspace*{-1ex}
	\subfigure[\hspace*{-3ex}]{
		\includegraphics[scale=0.175]{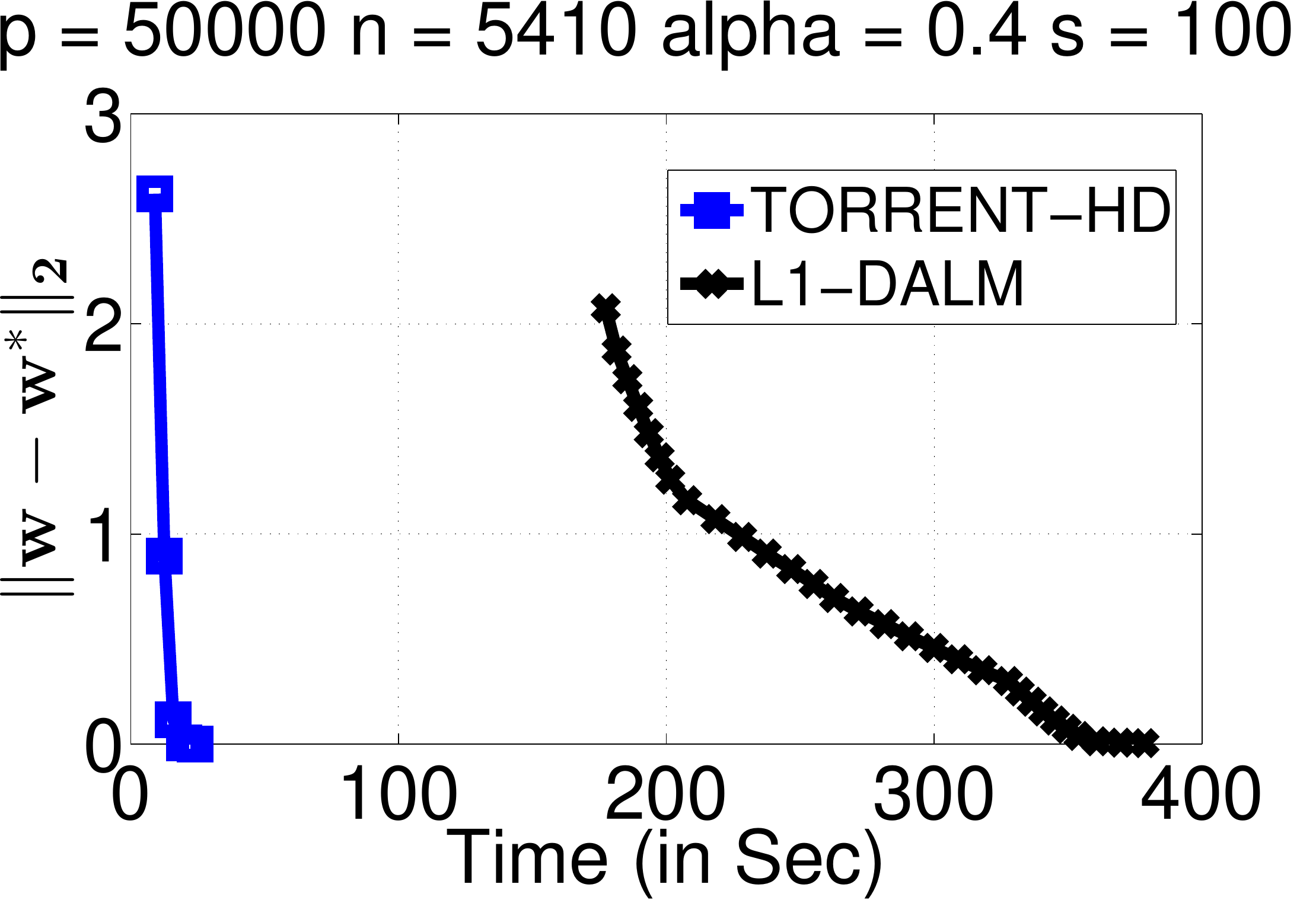}
	}
\hspace{2ex}
\caption{\small{(a), (b) and (c) Phase-transition diagrams depicting the recovery properties of the \alg-FC, \alg-HYB and $L_1$ algorithms. The colors red and blue represent a high and low probability of success resp. A method is considered successful in an experiment if it recovers $\bto$ upto a $10^{-4}$ relative error. Both variants of \alg can be seen to recover $\bto$ in presence of larger number of corruptions than the $L_1$ solver. (d) Variation in recovery error with the magnitude of corruption. As the corruption is increased, \alg-FC and \alg-HYB show improved performance while the problem becomes more difficult for the $L_1$ solver.}}
\label{fig:recoveryPlots}
\end{figure}

\textbf{Run Time}: In order to ascertain the recovery guarantees for \alg on ill-conditioned problems, we performed an experiment where data was sampled as $\x_i \sim \cN(\vz, \Sigma)$ where $\text{diag}(\Sigma)\sim U(0, 5)$. Figure ~\ref{fig:recoveryPlots} plots the recovery error as a function of time. \alg-HYB was able to correctly recover $\bto$ about $50\times$ faster than $L_1$-DALM which spent a considerable amount of time pre-processing the data matrix $X$. Even after allowing the $L_1$ algorithm to run for 500 iterations, it was unable to reach the desired residual error of $10^{-4}$. Figure \ref{fig:recoveryPlots} also shows that our \alg-HYB algorithm is able to converge to the optimal solution much faster than \alg-FC or \alg-GD. This is because \alg-FC solves a least square problem at each step and thus, even though it requires significantly fewer iterations to converge, each iteration in itself is very expensive. While each iteration of \alg-GD is cheap, it is still limited by the slow $\O{(1-\frac{1}{\kappa})^t}$ convergence rate of the gradient descent algorithm, where $\kappa$ is the condition number of the covariance matrix. \alg-HYB, on the other hand, is able to combine the strengths of both the methods to achieve faster convergence. 

\subsection{High Dimensional Results}
\quad\textbf{Recovery Property}: Figure \ref{fig:recoveryPlots} shows the variation in recovery error in the high-dimensional setting as the number of corrupted points was varied. For these experiments, $n$ was set to $5s^*\log(p)$ and the fraction of corrupted points $\alpha$ was varied from $0.1$ to $0.7$. While $L_1$-DALM fails to recover $\bto$ for $\alpha > 0.5$, \alg-HD offers perfect recovery even for $\alpha$ values upto 0.7.

\textbf{Run Time}: Figure \ref{fig:recoveryPlots} shows the variation in recovery error as a function of run time in this setting. $L_1$-DALM was found to be an order of magnitude slower than \alg-HD, making it infeasible for sparse high-dimensional settings. One key reason for this is that the $L_1$-DALM solver is significantly slower in identifying the set of clean points. For instance, whereas \alg-HD was able to identify the clean set of points in only $5$ iterations, it took $L_1$ around $250$ iterations to do the same.


\clearpage

\bibliographystyle{plain}
\bibliography{refs}

\appendix
\allowdisplaybreaks

\section{Convergence Guarantees with Dense Noise and Sparse Corruptions}
\label{app:dense-noise}
We will now present recovery guarantees for the \alg-FC algorithm when both, dense noise, as well as sparse adversarial corruptions are present. Extensions for \alg-GD and \alg-HYB will follow similarly.

\begin{thm}
\label{thm:fcht-dense-noise}
Let $X = \bs{\x_1, \dots, \x_n}\in \R^{p\times n}$ be the given data matrix and $\y=X^T\bto+\b+\vve$ be the corrupted output with sparse corruptions $\|\b\|_0\leq \alpha\cdot n$ as well as dense bounded noise $\vve$. Let Algorithm~\ref{algo:update-fc} be executed on this data with the thresholding parameter set to $\beta \geq \alpha$. Let $\Sigma_0$ be an invertible matrix such that $\Xt=\isSo X$ satisfies the SSC and SSS properties at level $\gamma$ with constants $\lambda_{\gamma}$ and $\Lambda_{\gamma}$ respectively (see Definition~\ref{defn:ssc-sss}). If the data satisfies $\frac{4\sqrt{\Lambda_{\beta}}}{\sqrt{\lambda_{1-\beta}}} < 1$, then after $t = \O{\log\br{\frac{1}{\sqrt n}\frac{\norm{\b}_2}{\epsilon}}}$ iterations, Algorithm~\ref{algo:update-fc} obtains an $\epsilon$-accurate solution $\btt$ i.e. $\norm{\btt - \bto}_2 \leq \epsilon + C\frac{\norm{\vve}_2}{\sqrt n}$ for some constant $C > 0$.
\end{thm}
\begin{proof}
We being by observing that the optimality of the model $\btn$ on the active set $S_t$ ensures
\[
\norm{\y_{S_t} - X_{S_t}^\top\btn}_2 = \norm{X_{S_t}^\top(\bto - \btn) + \vve_{S_t} + \b_{S_t}}_2 \leq \norm{\y_t - X_{S_t}^\top\bto}_2 = \norm{\vve_{S_t} + \b_{S_t}}_2,
\]
which, upon the application of the triangle inequality, gives us
\[
\norm{X_{S_t}^\top(\bto - \btn)}_2 \leq 2\norm{\vve_{S_t} + \b_{S_t}}_2.
\]
Since $\norm{X_{S_t}^\top(\bto - \btn)}_2 \geq \sqrt{\lambda_{1-\beta}}\norm{\bto - \btn}_2$, we get
\[
\norm{\bto - \btn}_2 \leq \frac{2}{\sqrt{\lambda_{1-\beta}}}\norm{\vve_{S_t} + \b_{S_t}}_2 \leq \frac{2}{\sqrt{\lambda_{1-\beta}}}\br{\norm{\vve}_2+\norm{\b_{S_t}}_2}.
\]
The hard thresholding step, on the other hand, guarantees that
\begin{align*}
\norm{X_{S_{t+1}}^\top(\bto - \btn) + \vve_{S_{t+1}} + \b_{S_{t+1}}}_2^2 &= \norm{\y_{S_{t+1}} - X_{S_{t+1}}^\top\btn}_2^2\\
																		 &\leq \norm{\y_{S_\ast} - X_{S_\ast}^\top\btn}_2\\
																		 &= \norm{X_{S_{\ast}}^\top(\bto - \btn) + \vve_{S_{\ast}}}_2^2.
\end{align*}
As before, let $\fa_{t+1} = S_{t+1}\backslash S_\ast$ and $\md_{t+1} = S_\ast\backslash S_{t+1}$. Then we have
\[
\norm{X_{\fa_{t+1}}^\top(\bto - \btn) + \vve_{\fa_{t+1}} + \b_{\fa_{t+1}}}_2 \leq \norm{X_{\md_{t+1}}^\top(\bto - \btn) + \vve_{\md_{t+1}}}_2.
\]
An application of the triangle inequality and the fact that $\norm{\b_{\fa_{t+1}}}_2 = \norm{\b_{S_{t+1}}}$ gives us
\begin{align*}
\norm{\b_{S_{t+1}}}_2 &\leq \norm{X_{\md_{t+1}}^\top(\bto - \btn)}_2 + \norm{X_{\fa_{t+1}}^\top(\bto - \btn)}_2 + \norm{\vve_{\fa_{t+1}}}_2 + \norm{\vve_{\md_{t+1}}}_2\\
					&\leq 2\sqrt{\Lambda_{\beta}}\norm{\bto - \btn}_2 + \sqrt{2}\norm{\vve}_2,\\
					&= \frac{4\sqrt{\Lambda_{\beta}}}{\sqrt{\lambda_{1-\beta}}}\norm{\b_{S_t}}_2 + (\frac{4\sqrt{\Lambda_{\beta}}}{\sqrt{\lambda_{1-\beta}}}+\sqrt 2)\norm{\vve}_2\\
					&\leq \eta\cdot\norm{\b_{S_t}}_2 + (1+\sqrt 2)\norm{\vve}_2,
\end{align*}
where the second step uses the fact that $\max\bc{\abs{\fa_{t+1}},\abs{\md_{t+1}}} \leq \beta\cdot n$ and the Cauchy-Schwartz inequality, and the last step uses the fact that for sufficiently small $\beta$, we have $\eta := \frac{4\sqrt{\Lambda_{\beta}}}{\sqrt{\lambda_{1-\beta}}}$. Using the inequality for $\norm{\btn-\bto}_2$ again gives us
\begin{align*}
\norm{\bto - \btn}_2 &\leq \frac{2}{\sqrt{\lambda_{1-\beta}}}\br{\norm{\vve}_2+\norm{\b_{S_t}}_2}\\
					 &\leq \frac{4+2\sqrt 2}{\sqrt{\lambda_{1-\beta}}}\norm{\vve}_2 + \frac{2\cdot\eta^t}{\sqrt{\lambda_{1-\beta}}}\norm{\b}_2
\end{align*}
For large enough $n$ we have $\sqrt{\lambda_{1-\beta}} \geq \O{\sqrt n}$, which completes the proof.
\end{proof}

Notice that for random Gaussian noise, this result gives the following convergence guarantee.

\begin{cor}
Let the date be generated as before with random Gaussian dense noise i.e. $\y=X^T\bto+\b+\vve$ with $\|\b\|_0\leq \alpha\cdot n$ and $\vve \sim \cN(\vz,\sigma^2\cdot I)$. Let Algorithm~\ref{algo:update-fc} be executed on this data with the thresholding parameter set to $\beta \geq \alpha$. Let $\Sigma_0$ be an invertible matrix such that $\Xt=\isSo X$ satisfies the SSC and SSS properties at level $\gamma$ with constants $\lambda_{\gamma}$ and $\Lambda_{\gamma}$ respectively (see Definition~\ref{defn:ssc-sss}). If the data satisfies $\frac{4\sqrt{\Lambda_{\beta}}}{\sqrt{\lambda_{1-\beta}}} < 1$, then after $t = \O{\log\br{\frac{1}{\sqrt n}\frac{\norm{\b}_2}{\epsilon}}}$ iterations, Algorithm~\ref{algo:update-fc} obtains an $\epsilon$-accurate solution $\btt$ i.e. $\norm{\btt - \bto}_2 \leq \epsilon + 2\sigma C$, where $C > 0$ is the constant in Theorem~\ref{thm:fcht-dense-noise}.
\end{cor}
\begin{proof}
Using tail bounds on Chi-squared distributions \cite{LaurentM00}, we get, with probability at least $1-\delta$,
\[
\norm{\vve}_2^2 \leq \sigma^2\br{n + 2\sqrt{n\log\frac{1}{\delta}} + 2\log\frac{1}{\delta}}.
\]
Thus, for $n > 4\log\frac{1}{\delta}$, we have $\norm{\vve}_2^2 \leq 2\sigma n$ which proves the result.
\end{proof}

\begin{remark}
We note that the design assumptions made by Theorem~\ref{thm:fcht-dense-noise} (i..e $\frac{4\sqrt{\Lambda_{\beta}}}{\sqrt{\lambda_{1-\beta}}} < 1$) are similar to those made by Theorem~\ref{thm:fcht} and would be satisfied with high probability by data sampled from sub-Gaussian distributions (see Appendix~\ref{sec:stat} for details).
\end{remark}

\section{Proof of Theorem~\ref{thm:fcht}}
\label{app:thm-fcht-proof}
\begin{repthm}{thm:fcht}
Let $X = \bs{\x_1, \dots, \x_n}\in \R^{p\times n}$ be the given data matrix and $\y=X^T\bto+\b$ be the corrupted output with $\|\b\|_0\leq \alpha\cdot n$. Let Algorithm~\ref{algo:update-fc} be executed on this data with the thresholding parameter set to $\beta \geq \alpha$. Let $\Sigma_0$ be an invertible matrix such that $\Xt=\isSo X$ satisfies the SSC and SSS properties at level $\gamma$ with constants $\lambda_{\gamma}$ and $\Lambda_{\gamma}$ respectively (see Definition~\ref{defn:ssc-sss}). If the data satisfies $\frac{(1+\sqrt 2)\Lambda_{\beta}}{\lambda_{1-\beta}} < 1$, then after $t = \O{\log\br{\frac{1}{\sqrt n}\frac{\norm{\b}_2}{\epsilon}}}$ iterations, Algorithm~\ref{algo:update-fc} obtains an $\epsilon$-accurate solution $\btt$ i.e. $\norm{\btt - \bto}_2 \leq \epsilon$.
\end{repthm}
\begin{proof}
Let $\r^t = \y - X^\top\btt$ be the vector of residuals at time $t$ and $C_t = X_{S_t}X_{S_t}^\top$. Since $\lambda_\alpha > 0$ (something which we shall establish later), we get
\[
\btn = C_t^{-1}X_{S_t}\y_{S_t} = C_t^{-1}X_{S_t}\br{X_{S_t}^\top\bto + \b_{S_t}} = \bto + C_t^{-1}X_{S_t}\b_{S_t}.
\]
Thus, for any set $S \subset [n]$, we have
\[
\r^{t+1}_S = \y_S - X_S^\top\w_{t+1} = \b_S - X_S^\top C_t^{-1}X_{S_t}\b_{S_t}
\]
This, gives us
\begin{align*}
\norm{\b_{S_{t+1}}}_2^2 &= \norm{\b_{S_{t+1}} - X_{S_{t+1}}^\top C_t^{-1}X_{S_t}\b_{S_t}}_2^2 - \norm{X_{S_{t+1}}^\top C_t^{-1}X_{S_t}\b_{S_t}}_2^2 + 2\cdot\b_{S_{t+1}}^\top X_{S_{t+1}}^\top C_t^{-1}X_{S_t}\b_{S_t}\\
&\stackrel{\zeta_1}{\leq} \norm{\b_{S_{\ast}} - X_{S_{\ast}}^\top C_t^{-1}X_{S_t}\b_{S_t}}_2^2 - \norm{X_{S_{t+1}}^\top C_t^{-1}X_{S_t}\b_{S_t}}_2^2 + 2\cdot\b_{S_{t+1}}^\top X_{S_{t+1}}^\top C_t^{-1}X_{S_t}\b_{S_t}\\
&\stackrel{\zeta_2}{=} \norm{X_{S_\ast}^\top C_t^{-1}X_{S_t}\b_{S_t}}_2^2 - \norm{X_{S_{t+1}}^\top C_t^{-1}X_{S_t}\b_{S_t}}_2^2 + 2\cdot\b_{S_{t+1}}^\top X_{S_{t+1}}^\top C_t^{-1}X_{S_t}\b_{S_t}\\
&\leq \norm{X_{S_\ast\backslash S_{t+1}}^\top C_t^{-1}X_{S_t}\b_{S_t}}_2^2 + 2\cdot\b_{S_{t+1}}^\top X_{S_{t+1}}^\top C_t^{-1}X_{S_t}\b_{S_t}\\
&\stackrel{\zeta_3}{=}\norm{\Xt_{S_\ast\backslash S_{t+1}}^\top \left(\Xt_{S_t}\Xt_{S_t}^T\right)^{-1}\Xt_{S_t}\b_{S_t}}_2^2 + 2\cdot\b_{S_{t+1}}^\top \Xt_{S_{t+1}}^\top \left( \Xt_{S_t}\Xt_{S_t}^T\right)^{-1} \Xt_{S_t}\b_{S_t}\\
&\stackrel{\zeta_4}{\leq} {\frac{\Lambda_{\beta}^2}{\lambda_{1-\beta}^2}}\cdot\norm{\b_{S_t}}_2^2 + 2\cdot\frac{\Lambda_{\beta}}{\lambda_{1-\beta}}\cdot\norm{\b_{S_t}}_2\norm{\b_{S_{t+1}}}_2,
\end{align*}
where $\zeta_1$ follows since the hard thresholding step ensures $\norm{\r^{t+1}_{S_{t+1}}}_2^2 \leq \norm{\r^{t+1}_{S_\ast}}_2^2$ (see Claim~\ref{clm:supp-ord-stat} and use the fact that $\beta \geq \alpha$), $\zeta_2$ notices the fact that $\b_{S_\ast} = \vz$. $\zeta_3$ follows from setting $\Xt=\isSo X$ and $X_S^\top C_t^{-1}X_{S'}=\Xt_S^\top(\Xt_{S_t}\Xt_{S_t}^\top)^{-1}\Xt_{S'}$. $\zeta_4$ follows from the definition of SSC and SSS properties, $\norm{\b_{S_t}}_0 \leq \norm{\b}_0 \leq \beta\cdot n$ and $\abs{S_\ast\backslash S_{t+1}} \leq \beta\cdot n$. Solving the quadratic equation gives us
\begin{align}
\label{eq:fcht-rate}
\norm{\b_{S_{t+1}}}_2 \leq (1 + \sqrt 2)\cdot\frac{\Lambda_{\beta}}{\lambda_{1-\beta}}\cdot\norm{\b_{S_{t}}}_2.
\end{align}
Let $\eta := \frac{(1 + \sqrt 2)\Lambda_{\beta}}{\lambda_{1-\beta}}$ denote the convergence rate in \eqref{eq:fcht-rate}. We shall show below that for a large family of random designs, we have $\eta < 1$ if $n \geq \Om{p + \log\frac{1}{\delta}}$. We now recall from our earlier discussion that $\btn = \bto + C_t^{-1}X_{S_t}\b_{S_t}$ which gives us
\[
\norm{\btn - \bto}_2 = \norm{C_t^{-1}X_{S_t}\b_{S_t}}_2 \leq \frac{\sqrt{\Lambda_\beta}}{\lambda_{1-\beta}}\cdot\norm{\b_{S_t}}_2 \leq \eta^t\cdot\frac{\sqrt{\Lambda_\beta}}{\lambda_{1-\beta}}\norm{\b}_2 \leq \epsilon,
\]
for $t \geq \log_{\frac{1}{\eta}}\br{\frac{\sqrt{\Lambda_\beta}}{\lambda_{1-\beta}}\cdot\frac{\norm{\b}_2}{\epsilon}}$. Noting that $\frac{\sqrt{\Lambda_\beta}}{\lambda_{1-\beta}} \leq \O{\frac{1}{\sqrt n}}$ establishes the convergence result.
\end{proof}

\section{Proof of Theorem~\ref{thm:fcht-explicit-rate}}
\label{app:thm-fcht-explicit-rate}
\begin{repthm}{thm:fcht-explicit-rate}
Let $X=[\x_1, \dots, \x_n]\in \R^{p\times n}$ be the given data matrix with each $\x_i \sim \cN(\vz	, \Sigma)$. Let $\y=X^\top\bto+\b$ and $\|\b\|_0\leq \alpha \cdot n$. Also, let $\alpha \leq \beta < \frac{1}{65}$ and $n \geq \Om{p + \log\frac{1}{\delta}}$. Then, with probability at least $1-\delta$, the data satisfies $\frac{(1+\sqrt 2)\Lambda_{\beta}}{\lambda_{1-\beta}} < \frac{9}{10}$. More specifically, after $T \geq 10\log\br{\frac{1}{\sqrt n}\frac{\norm{\b}_2}{\epsilon}}$ iterations of Algorithm~\ref{algo:fcht} with the thresholding parameter set to $\beta$, we have $\norm{\bt^T-\bto}\leq \epsilon.$
\end{repthm}
\begin{proof}
We note that whenever $\x \sim \cN(\vz,\Sigma)$ then $\isS\x \sim \cN(\vz,I)$. Thus, Theorem~\ref{thm:ev-bound-gaussian-local} assures us that with probability at least $1 - \delta$, the data matrix $\Xt = \isS X$ satisfies the SSC and SSS properties with the following constants
\begin{align*}
\Lambda_\beta &\leq \beta n\br{1 + 3e\sqrt{6\log\frac{e}{\beta}}} + \O{\sqrt{np + n\log\frac{1}{\delta}}}\\
\lambda_{1-\beta} &\geq n - \beta n\br{1 + 3e\sqrt{6\log\frac{e}{\beta}}} - \Om{\sqrt{np + n\log\frac{1}{\delta}}}
\end{align*}
Thus, the convergence given be Algorithm~\ref{algo:fcht}, when invoked with $\Sigma_0 = \Sigma$, relies on the quantity $\eta = \frac{(1 + \sqrt 2)\Lambda_{\beta}}{\lambda_{1-\beta}}$ being less than unity. This translates to the requirement $(1 + \sqrt 2)\Lambda_\beta \leq \lambda_{1-\beta}$. Using the above bounds translates that requirement to
\[
\underbrace{\vphantom{\O{\sqrt{\frac{p}{n} + \frac{1}{n}\log\frac{1}{\delta}}}}(2+\sqrt 2){\beta\br{1 + 3e\sqrt{6\log\frac{e}{\beta}}}}}_{(A)} + \underbrace{\O{\sqrt{\frac{p}{n} + \frac{1}{n}\log\frac{1}{\delta}}}}_{(B)} < 1.
\]
For $n = \Om{p + \log\frac{1}{\delta}}$, the second quantity $(B)$ can be made as small a constant as necessary. Tackling the first quantity $(A)$ turns out to be more challenging. However, we can show that for all $\beta < \frac{1}{190}$, we get $\eta = \frac{(1 + \sqrt 2)\Lambda_{\beta}}{\lambda_{1-\beta}} < \frac{9}{10}$ which establishes the claimed result. Thus, Algorithm~\ref{algo:fcht} can tolerate a corruption index of upto $\alpha \leq \frac{1}{190}$. However, we note that using a more finely tuned setting of the constant $\epsilon$ in the proof of Theorem~\ref{thm:ev-bound-gaussian-local} and a more careful proof using tight tail inequalities for chi-squared distributions \cite{LaurentM00}, we can achieve a better corruption level tolerance of $\alpha < \frac{1}{65}$.
\end{proof}

\section{Proof of Theorem~\ref{thm:fcht-grades}}
\label{app:thm-fcht-grades}
\begin{repthm}{thm:fcht-grades}
Let $X = \bs{\x_1, \dots, \x_n}\in \R^{p\times n}$ be the given data matrix and $\y=X^T\bto+\b$ be the corrupted output with $\|\b\|_0\leq \alpha\cdot n$. Let $X$ satisfy the SSC and SSS properties at level $\gamma$ with constants $\lambda_{\gamma}$ and $\Lambda_{\gamma}$ respectively (see Definition~\ref{defn:ssc-sss}). Let Algorithm~\ref{algo:fcht} be executed on this data with the GD update (Algorithm~\ref{algo:update-grades}) with the thresholding parameter set to $\beta \geq \alpha$ and the step length set to $\eta = \frac{1}{\Lambda_{1-\beta}}$.  If the data satisfies $\max\bc{\eta\sqrt{\Lambda_\beta}, 1 - \eta\lambda_{1-\beta}} \leq \frac{1}{4}$, then after $t = \O{\log\br{\frac{\norm{b}_2}{\sqrt{n}}\frac{1}{\epsilon}}}$ iterations, Algorithm~\ref{algo:fcht} obtains an $\epsilon$-accurate solution $\btt$ i.e. $\norm{\btt - \bto}_2 \leq \epsilon$.
\end{repthm}
\begin{proof}
Let $\r^t = \y - X^\top\btt$ be the vector of residuals at time $t$ and $C_t = X_{S_t}X_{S_t}^\top$. We have
\[
\btn = \btt + \eta\cdot X_{S_t}\r^t_{S_t} = \btt + \eta\cdot X_{S_t}(\y_{S_t} - X_{S_t}^\top\btt)
\]
The thresholding step ensures that $\norm{\r^{t+1}_{S_{t+1}}}_2^2 \leq \norm{\r^{t+1}_{S_\ast}}_2^2$ (see Claim~\ref{clm:supp-ord-stat} and use $\beta \geq \alpha$) which implies
\[
\norm{\r^{t+1}_{\fa_{t+1}}}_2^2 \leq \norm{\r^{t+1}_{\md_{t+1}}}_2^2,
\]
where $\fa_{t+1} = S_{t+1}\backslash S_\ast$ are the \emph{corrupted recoveries} and $\md_{t+1} = S_\ast\backslash S_{t+1}$ are the clean points \emph{missed} out from \emph{detection}. Note that $\abs{\fa_{t+1}} \leq \alpha\cdot n$ and $\abs{\md_{t+1}} \leq \beta\cdot n$. Since $\b_{S_\ast} = \vz$ and $\md_{t+1} \subseteq S_\ast$, we get
\[
\norm{\b_{\fa_{t+1}} + X_{\fa_{t+1}}^\top(\bto - \btn)}_2 \leq \norm{X_{\md_{t+1}}^\top(\bto - \btn)}_2
\]
Using the SSS conditions and the fact that $\norm{\b_{S_{t+1}}}_2 = \norm{\b_{S_{t+1}\backslash S_\ast}}_2$ gives us
\[
\norm{\b_{S_{t+1}}}_2 = \norm{\b_{\fa_{t+1}}}_2 \leq (\sqrt{\Lambda_{\alpha}} + \sqrt{\Lambda_{\beta}})\norm{\bto - \btn}_2 \leq 2\sqrt{\Lambda_\beta}\norm{\bto - \btn}_2
\]
Now, using the expression for $\btn$ gives us
\[
\norm{\bto - \btn}_2 \leq \norm{(I - \eta C_t)(\bto - \btt)}_2 + \eta\norm{X_{S_t}\b_{S_t}}_2
\]
We will bound the two terms on the right hand separately. We can bound the second term easily as
\[
\eta\norm{X_{S_t}\b_{S_t}}_2 \leq \eta\sqrt{\Lambda_\alpha}\norm{\b_{S_t}}_2 \leq \eta\sqrt{\Lambda_\beta}\norm{\b_{S_t}}_2,
\]
since $\norm{\b_{S_t}}_0 \leq \alpha\cdot n$. For the first term we observe that for $\eta \leq \frac{1}{\Lambda_{1-\beta}}$, we have
\[
\norm{I - \eta C_t}_2 = \underset{\v \in S^{p-1}}\sup\abs{1 - \eta\cdot\v^\top C_t\v} = \underset{\v \in S^{p-1}}\sup\bc{1 - \eta\cdot\v^\top C_t\v} \leq 1 - \eta\lambda_{1-\beta},
\]
which we can use to bound
\[
\norm{\bto - \btn}_2 \leq (1-\eta\lambda_{1-\beta})\norm{\bto - \btt}_2 + \eta\sqrt{\Lambda_\beta}\norm{\b_{S_t}}_2
\]
This gives us, for $\eta = \frac{1}{\Lambda_{1-\beta}}$,
\[
\norm{\b_{S_{t+1}}}_2 \leq 2\sqrt{\Lambda_\beta}\norm{\bto - \btn}_2 \leq 2\underbrace{\vphantom{\frac{\Lambda_\beta}{\Lambda_{1-\beta}}}\br{1-\frac{\lambda_{1-\beta}}{\Lambda_{1-\beta}}}}_{(P)}\sqrt{\Lambda_\beta}\norm{\bto - \btt}_2 + 2\underbrace{\frac{\Lambda_\beta}{\Lambda_{1-\beta}}}_{(Q)}\norm{\b_{S_t}}_2.
\]
For Gaussian designs and small enough $\beta$, we can show $(Q) \leq \frac{1}{4}$ as we did in Theorem~\ref{thm:fcht-explicit-rate}. To bound $(P)$, we use the lower bound on $\lambda_{1-\beta}$ given by Theorem~\ref{thm:ev-bound-gaussian-local} and use the following tighter upper bound for $\Lambda_{1-\beta}$:
\[
\Lambda_{1-\beta} \leq \br{(1-\beta) + 3e\sqrt{6\beta(1-\beta)\log\frac{e}{\beta}}}n + \O{\sqrt{np + n\log\frac{1}{\delta}}}
\]
The above bound is obtained similarly to the one in Theorem~\ref{thm:ev-bound-gaussian-local} but uses the identity $\binom{n}{k} = \binom{n}{n-k} \leq \br{\frac{en}{n-k}}^{n-k}$ for values of $k \geq n/2$ instead. For small enough $\beta$ and $n = \Om{\kappa^2(\Sigma)(p + \log\frac{1}{\delta})}$, we can then show $(P) \leq \frac{1}{4}$ as well. Let $\Psi_t := \sqrt{n}\norm{\bto - \btt}_2 + \norm{b_{S_t}}$. Using elementary manipulations and the fact that $\sqrt{\Lambda_\beta} \geq \Om{\sqrt{n}}$, we can then show that
\[
\Psi_{t+1} \leq 3/4\cdot\Psi_t.
\]
Thus, in $t = \O{\log\br{\br{\norm{\bto}_2 + \frac{\norm{b}_2}{\sqrt{n}}}\frac{1}{\epsilon}}}$ iterations of the algorithm, we arrive at an $\epsilon$-optimal solution i.e. $\norm{\bto-\btt}_2 \leq \epsilon$. A similar argument holds true for sub-Gaussian designs as well.
\end{proof}

\section{Proof of Theorem~\ref{thm:fcht-hyb-rate}}
\label{app:thm-fcht-hyb-rate}
\begin{repthm}{thm:fcht-hyb-rate}
Suppose Algorithm~\ref{algo:update-hybrid} is executed on data that allows Algorithms~\ref{algo:update-fc}~and~\ref{algo:update-grades} a convergence rate of $\eta_\fc$ and $\eta_\gd$ respectively. Suppose we have $2\cdot\eta_\fc\cdot\eta_\gd < 1$. Then for \emph{any} interleavings of the FC and GD steps that the policy may enforce, after $t = \O{\log\br{\frac{1}{\sqrt n}\frac{\norm{\b}_2}{\epsilon}}}$ iterations, Algorithm~\ref{algo:update-hybrid} ensures an $\epsilon$-optimal solution i.e. $\norm{\btt - \bto} \leq \epsilon$.
\end{repthm}
\begin{proof}
Our proof shall essentially show that the FC and GD steps do not undo the progress made by the other if executed in succession and if $2\cdot\eta_\fc\cdot\eta_\gd < 1$, actually ensure non-trivial progress. Let
\begin{align*}
\Psi^\fc_t &= \norm{\b_{S_t}}_2\\
\Psi^\gd_t &= \sqrt{n}\norm{\btt - \bto} + \norm{\b_{S_t}}_2
\end{align*}
denote the potential functions used in the analyses of the FC and GD algorithms before. Then we will show below that if the FC and GD algorithms are executed in steps $t$ and $t+1$ then we have
\[
\Psi^\fc_{t+2} \leq 2\cdot\eta_\fc\cdot\eta_\gd\cdot\Psi^\fc_t
\]
Alternatively, if the GD and FC algorithms are executed in steps $t$ and $t+1$ respectively, then
\[
\Psi^\gd_{t+2} \leq 2\cdot\eta_\fc\cdot\eta_\gd\cdot\Psi^\gd_t
\]
Thus, if algorithm executes the FC step at the time step $t$, then it would at least ensure $\Psi^\fc_t \leq \br{2\cdot\eta_\fc\cdot\eta_\gd}^{t/2}\cdot\Psi^\fc_0$ (similarly if the last step is a GD step). Since both the FC and GD algorithms ensure $\norm{\btt - \bto}_2 \leq \epsilon$ for $t \geq \O{\log\br{\frac{1}{\sqrt n}\frac{\norm{b}_2}{\epsilon}}}$, the claim would follow.

We now prove the two claimed results regarding the two types of interleaving below
\begin{enumerate}
	\item $\fc \longrightarrow \gd$\newline
	The FC step guarantees $\norm{\b_{S_{t+1}}}_2 \leq \eta_\fc\cdot\norm{\b_{S_t}}$ as well as $\norm{\btn-\bto}_2 \leq \eta_\fc\cdot\frac{\norm{\b_{S_t}}}{\sqrt n}$, whereas the GD step guarantees $\Psi^\gd_{t+2} \leq \eta_\gd\cdot\Psi^\gd_{t+1}$. Together these guarantee
	\begin{align*}
	\sqrt{n}\norm{\btnn - \bto}_2 + \norm{\b_{S_{t+2}}}_2 &\leq \eta_\gd\cdot{\sqrt{n}\norm{\btn - \bto}_2 + \norm{\b_{S_{t+1}}}_2}\\
														  &\leq 2\cdot\eta_\fc\cdot\eta_\gd\cdot\norm{\b_{S_t}}_2
	\end{align*}
	Since $\sqrt{n}\norm{\btnn - \bto}_2 \geq 0$, this yields the result.
	\item $\gd \longrightarrow \fc$\newline
	The GD step guarantees $\Psi^\gd_{t+1} \leq \eta_\gd\cdot\Psi^\gd_{t}$ whereas the FC step guarantees $\norm{\b_{S_{t+2}}}_2 \leq \eta_\fc\cdot\norm{\b_{S_{t+1}}}$ as well as $\norm{\btnn-\bto}_2 \leq \eta_\fc\cdot\frac{\norm{\b_{S_{t+1}}}}{\sqrt n}$. Together these guarantee
	\begin{align*}
	\sqrt{n}\norm{\btnn - \bto}_2 + \norm{\b_{S_{t+2}}}_2 &\leq 2\eta_\fc\norm{\b_{S_{t+1}}}_2\\
														  &\leq 2\cdot\eta_\fc\cdot\eta_\gd\cdot\Psi^{\gd}_t,
	\end{align*}
	where the second step follows from the GD step guarantee since $\sqrt{n}\norm{\btn - \bto}_2 \geq 0$.
\end{enumerate}
This finishes the proof.
\end{proof}

\section{Proof of Theorem~\ref{thm:fcht-highd}}
\label{app:thm-fcht-highd}
\begin{repthm}{thm:fcht-highd}
Let $X = \bs{\x_1, \dots, \x_n}\in \R^{p\times n}$ be the given data matrix and $\y=X^T\bto+\b$ be the corrupted output with $\norm{\bto}_0 \leq s^\ast$ and $\|\b\|_0\leq \alpha\cdot n$. Let Algorithm~\ref{algo:update-fc} be executed on this data with the IHT update from \cite{JainTK14} and thresholding parameter set to $\beta \geq \alpha$. Let $\Sigma_0$ be an invertible matrix such that $\isSo X$ satisfies the SRSC and SRSS properties at level $(\gamma,2s+s^\ast)$ with constants $\alpha_{(\gamma,2s+s^\ast)}$ and $L_{(\gamma,2s+s^\ast)}$ respectively (see Definition~\ref{defn:srsc-srss}) for $s \geq 32\br{\frac{L_{(\gamma,2s+s^\ast)}}{\alpha_{(\gamma,2s+s^\ast)}}}$ with $\gamma = 1-\beta$. If $X$ also satisfies $\frac{4L_{(\beta,s+s^\ast)}}{\alpha_{(1-\beta,s+s^\ast)}} < 1$, then after $t = \O{\log\br{\frac{1}{\sqrt n}\frac{\norm{\b}_2}{\epsilon}}}$ iterations, Algorithm~\ref{algo:update-fc} obtains an $\epsilon$-accurate solution $\btt$ i.e. $\norm{\btt - \bto}_2 \leq \epsilon$. In particular, if $X$ is sampled from a Gaussian distribution $\cN(\vz,\Sigma)$ and $n \geq \Om{{(2s+s^\ast)\log p + \log\frac{1}{\delta}}}$, then for all values of $\alpha \leq \beta < \frac{1}{65}$, we can guarantee recovery as $\norm{\btt - \bto}_2 \leq \epsilon$.
\end{repthm}
\begin{proof}
We first begin with the guarantee provided by existing sparse recovery techniques. The results of \cite{JainTK14}, for example, indicate that if the input to the algorithm indeed satisfies the RSC and RSS properties at the level $(1-\beta,2s+s^\ast)$ with constants $\alpha_{2s+s^\ast}$ and $L_{2s+s^\ast}$ for $s \geq 32\br{\frac{L_{2s+s^\ast}}{\alpha_{2s+s^\ast}}}$, then in time $\tau = \O{\frac{L_{2s+s^\ast}}{\alpha_{2s+s^\ast}}\cdot\log\br{\frac{\norm{b}_2}{\rho}}}$, the IHT algorithm \cite[Algorithm 1]{JainTK14} outputs an updated model $\btn$ that satisfies $\norm{\btn}_0 \leq s$, as well as
\[
\norm{\y_{S_t} - X_{S_t}^\top\btn}_2^2 \leq \norm{\y_{S_t} - X_{S_t}^\top\bto}_2^2 + \rho.
\]
We will set $\rho$ later. Since the SRSC and SRSS properties ensure the above and $\y = X^\top\bto + \b$, this gives us
\[
\norm{X_{S_t}^\top(\btn - \bto)}_2^2 \leq 2(\btn-\bto)^\top X_{S_t}^\top\b_{S_t} + \rho = 2(\btn-\bto)^\top X_{S_t\cap \bar{S}_\ast}^\top\b_{S_t\cap \bar{S}_\ast} + \rho,
\]
since $\b_S = \vz$ for any set $S \cap \bar{S}_\ast = \phi$. We now analyze the two sides separately below using the SRSC and SRSS properties below. For any $S \subset [n]$, denote $\tilde X_{S} := \isSo X$.
\begin{align*}
\norm{X_{S_t}^\top(\btn - \bto)}_2^2 &= \norm{\tilde X_{S_t}^\top\sSo(\btn - \bto)}_2^2 \geq \alpha_{(1-\beta,s+s^\ast)}\norm{\sSo(\btn-\bto)}_2^2\\
\norm{X_{S_t\cap \bar{S}_\ast}(\btn-\bto)} &= \norm{\tilde X_{S_t\cap \bar{S}_\ast}\sSo(\btn-\bto)} \leq \sqrt{L_{(\beta,s+s^\ast)}}\norm{\sSo(\btn-\bto)}_2.
\end{align*}
Now, if $\norm{\btn-\bto}_2 \geq \epsilon$, then $\norm{\sSo(\btn-\bto)}_2 \geq \sqrt{\lambda_{\min}(\Sigma_0)}\cdot\epsilon$. This give us
\begin{align*}
\norm{\sSo(\btn-\bto)}_2 &\leq \frac{2\sqrt{L_{(\beta,s+s^\ast)}}}{\alpha_{(1-\beta,s+s^\ast)}}\norm{\b_{S_t\cap \bar{S}_\ast}}_2 + \frac{\rho}{\alpha_{(1-\beta,s+s^\ast)}}\\
						 &= \frac{2\sqrt{L_{(\beta,s+s^\ast)}}}{\alpha_{(1-\beta,s+s^\ast)}}\norm{\b_{S_t}}_2  + \frac{\rho}{\epsilon\cdot\sqrt{\lambda_{\min}(\Sigma_0)}\cdot\alpha_{(1-\beta,s+s^\ast)}}.
\end{align*}
We note that although we declared the SRSC and SRSS properties for the action of matrices on sparse vectors (such as $\bto -\btn$), we instead applied them above to the action of matrices on sparse vectors transformed by $\sSo$ ($\sSo(\bto - \btn)$). Since $\sSo\v$ need not be sparse even if $\v$ is sparse, this appears to pose a problem. However, all we need to resolve this is to notice that the proof technique of Theorem~\ref{thm:ev-bound-high-d-gaussian-local} which would be used to establish the SRSC and SRSS properties, holds in general for not just the action of a matrix on the set of sparse vectors, but on vectors in the union of any fixed set of low dimensional subspaces.

More specifically, we can modify the RSC and RSS properties (and by extension, the SRSC and SRSS properties), to requiring that the matrix $X$ act as an approximate isometry on the following set of vectors $S^{p-1}_{(s,\Sigma_0)} := \bc{\v: \v = \isSo \v' \text{ for some } \v' \in S^{p-1}_s}$. We refer the reader to the work of \cite{Blumensath11} which describes this technique in great detail. Proceeding with the proof, the assurance of the thresholding step, as used in the proof of Theorem~\ref{thm:fcht-grades}, along with a straightforward application of the (modified) SRSS property gives us
\begin{align*}
\norm{\b_{S_{t+1}}}_2 &\leq \norm{X_{\fa_{t+1}}^\top(\btn-\bto)}_2 + \norm{X_{\md_{t+1}}^\top(\btn-\bto)}_2\\
					  &= \norm{\tilde X_{\fa_{t+1}}^\top\sSo(\btn-\bto)}_2 + \norm{\tilde X_{\md_{t+1}}^\top\sSo(\btn-\bto)}_2\\
					  &\leq 2\sqrt{L_{(\beta,s+s^\ast)}}\norm{\sSo(\btn-\bto)}_2\\
					  &\leq \frac{4L_{(\beta,s+s^\ast)}}{\alpha_{(1-\beta,s+s^\ast)}}\norm{\b_{S_t}}_2  + \frac{2\rho\sqrt{L_{(\beta,s+s^\ast)}}}{\epsilon\cdot\sqrt{\lambda_{\min}(\Sigma_0)}\cdot\alpha_{(1-\beta,s+s^\ast)}}
\end{align*}
Thus, whenever $\norm{\btn-\bto}_2 > \epsilon$, in successive steps, $\norm{\b_{S_t}}_2$ undergoes a linear decrease. Denoting $\eta := \frac{4L_{(\beta,s+s^\ast)}}{\alpha_{(1-\beta,s+s^\ast)}}$, we get
\[
\norm{\b_{S_{t+1}}}_2 \leq \eta^t\cdot\norm{\b}_2 + \br{\frac{1-\eta^t}{1-\eta}}\frac{2\rho\sqrt{L_{(\beta,s+s^\ast)}}}{\epsilon\cdot\sqrt{\lambda_{\min}(\Sigma_0)}\cdot\alpha_{(1-\beta,s+s^\ast)}}
\]
 and using $\norm{\sSo(\btt-\bto)}_2 \geq \sqrt{\lambda_{\min}(\Sigma_0)}\norm{\btt-\bto}_2$ gives us
\begin{align*}
\norm{\btn-\bto}_2 &\leq \frac{2\sqrt{L_{(\beta,s+s^\ast)}}}{\sqrt{\lambda_{\min}(\Sigma_0)}\cdot\alpha_{(1-\beta,s+s^\ast)}}\norm{\b_{S_{t+1}}}_2  + \frac{\rho}{{\lambda_{\min}(\Sigma_0)}\cdot\alpha_{(1-\beta,s+s^\ast)}}\\
				   &\leq \eta^t\frac{2\sqrt{L_{(\beta,s+s^\ast)}}}{\sqrt{\lambda_{\min}(\Sigma_0)}\cdot\alpha_{(1-\beta,s+s^\ast)}}\norm{\b}_2 + \frac{36\rho}{\epsilon\cdot{\lambda_{\min}(\Sigma_0)}\cdot\alpha_{(1-\beta,s+s^\ast)}},
\end{align*}
where we have assumed that $\frac{4L_{(\beta,s+s^\ast)}}{\alpha_{(1-\beta,s+s^\ast)}} < 9/10$, something that we shall establish below. Note that $\lambda_{\min}(\Sigma_0) >0$ since $\Sigma$ is assumed to be invertible. In the random design settings we shall consider, we also have $\frac{\sqrt{L_{(\beta,s+s^\ast)}}}{\sqrt{\lambda_{\min}(\Sigma_0)}\cdot\alpha_{(1-\beta,s+s^\ast)}} = \O{\frac{1}{\sqrt n}}$. Then setting $\rho \leq \frac{1}{72}\epsilon^2\cdot{\lambda_{\min}(\Sigma_0)}\cdot{\alpha_{(1-\beta,s+s^\ast)}}$ proves the convergence result.

As before, we can use the above result to establish sparse recovery guarantees in the statistical setting for Gaussian and sub-Gaussian design models. If our data matrix $X$ is generated from a Gaussian distribution $\cN(\vz,\Sigma)$ for some invertible $\Sigma$, then the results in Theorem~\ref{thm:ev-bound-high-d-gaussian-local} can be used to establish that $\isS X$ satisfies the SRSC and SRSS properties at the required levels and that for $\alpha < \frac{1}{190}$ and $n \geq \Om{{(2s+s^\ast)\log p + \log\frac{1}{\delta}}}$, we have $\eta = \frac{2L_{(\beta,s+s^\ast)}}{\alpha_{(1-\beta,s+s^\ast)}} < 9/10$.

Thus, the above result can be applied with $\Sigma_0 = \Sigma$ to get convergence guarantees in the general Gaussian setting. We note that the above analysis can tolerate the same level of corruption as Theorem~\ref{thm:fcht-explicit-rate} and thus, we can improve the noise tolerance level to $\alpha \leq \frac{1}{65}$ here as well. We also note that these results can be readily extended to the sub-Gaussian setting as well.
\end{proof}

\section{Robust Statistical Estimation}
\label{sec:stat}
This section elaborates on how results on the convergence guarantees of our algorithms can be used to give guarantees for robust statistical estimation problems. We begin with a few definition of sampling models that would be used in our results.

\begin{defn}
A random variable $x \in \R$ is called sub-Gaussian if the following quantity is finite
\[
\underset{p \geq 1}{\sup}\ p^{-1/2}\br{\Ebb\abs{x}^p}^{1/p}.
\]
Moreover, the smallest upper bound on this quantity is referred to as the sub-Gaussian norm of $x$ and denoted as $\normsg{x}$.
\end{defn}

\begin{defn}
A vector-valued random variable $\x \in \R^p$ is called sub-Gaussian if its unidimensional marginals $\ip{\x}{\v}$ are sub-Gaussian for all $\v \in S^{p-1}$. Moreover, its sub-Gaussian norm is defined as follows
\[
\normsg{X} :=  \underset{\v \in S^{p-1}}{\sup} \normsg{\ip{\x}{\v}}
\]
\end{defn}

We will begin with the analysis of Gaussian designs and then extend our analysis for the class of general sub-Gaussian designs.

\begin{lem}
\label{lem:ev-bound-gaussian-global}
Let $X \in \R^{p \times n}$ be a matrix whose columns are sampled i.i.d from a standard Gaussian distribution i.e. $\x_i \sim \cN(\vz,I)$. Then for any $\epsilon > 0$, with probability at least $1 - \delta$, $X$ satisfies
\begin{align*}
s_{\max}(XX^\top) &\leq n + (1 - 2\epsilon)^{-1}\sqrt{cnp + c'n\log\frac{2}{\delta}}\\
s_{\min}(XX^\top) &\geq n - (1 - 2\epsilon)^{-1}\sqrt{cnp + c'n\log\frac{2}{\delta}},
\end{align*}
where $c = 24e^2\log\frac{3}{\epsilon}$ and $c' = 24e^2$.
\end{lem}
\begin{proof}
We will first use the fact that $X$ is sampled from a standard Gaussian to show that its covariance concentrates around identity. Thus, we first show that with high probability,
\[
\norm{XX^\top - nI}_2 \leq \epsilon_1
\]
for some $\epsilon_1 < 1$. Doing so will automatically establish the following result
\[
n - \epsilon_1 \leq s_{\min}(XX^\top) \leq s_{\max}(XX^\top) \leq n + \epsilon_1.
\]
Let $A := XX^\top - I$. We will use the technique of covering numbers \cite{Vershynin12} to establish the above. Let $\cC^{p-1}(\epsilon) \subset S^{p-1}$ be an $\epsilon$ cover for $S^{p-1}$ i.e. for all $\u \in S^{p-1}$, there exists at least one $\v \in \cC^{p-1}$ such that $\norm{\u - \v}_2 \leq \epsilon$. Standard constructions \cite[see Lemma 5.2]{Vershynin12} guarantee such a cover of size at most $\br{1 + \frac{2}{\epsilon}}^p \leq \br{\frac{3}{\epsilon}}^p$. Now for any $\u \in S^{p-1}$ and $\v \in \cC^{p-1}$ such that $\norm{\u-\v}_2 \leq \epsilon$, we have
\[
\abs{\u^\top A\u - \v^\top A\v} \leq \abs{\u^\top A(\u - \v)} + \abs{\v^\top A(\u - \v)} \leq 2\epsilon\norm{A}_2,
\]
which gives us
\[
\norm{XX^\top - nI}_2 \leq (1-2\epsilon)^{-1}\cdot\underset{\v \in \cC^{p-1}(\epsilon)}{\sup}\abs{\norm{X^\top\v}_2^2 - n}.
\]
Now for a fixed $\v \in S^{n-1}$, the random variable $\norm{X^\top\v}_2^2$ is distributed as a $\chi^2(n)$ distribution with $n$ degrees of freedom. Using Lemma~\ref{lem:chi-2-bernstein}, we get, for any $\mu < 1$,
\[
\Pr{\abs{\norm{X^\top\v}_2^2 - n}\geq \mu n} \leq 2\exp\br{-\min\bc{\frac{\mu^2n^2}{24ne^2},\frac{\mu n}{4\sqrt 3e}}} \leq 2\exp\br{-\frac{\mu^2n}{24e^2}}.
\]
Setting $\mu^2 = c\cdot\frac{p}{n} + c'\cdot\frac{\log\frac{2}{\delta}}{n}$, where $c = 24e^2\log\frac{3}{\epsilon}$ and $c' = 24e^2$, and taking a union bound over all $\cC^{p-1}(\epsilon)$, we get
\[
\Pr{\underset{\v \in \cC^{p-1}(\epsilon)}{\sup}\abs{\norm{X^\top\v}_2^2 - n}\geq \sqrt{cnp + c'n\log\frac{2}{\delta}}} \leq 2\br{\frac{3}{\epsilon}}^p\exp\br{-\frac{\mu^2n}{24e^2}} \leq \delta.
\]
This implies that with probability at least $1 - \delta$,
\[
\norm{XX^\top - nI}_2 \leq (1 - 2\epsilon)^{-1}\sqrt{cnp + c'n\log\frac{2}{\delta}},
\]
which gives us the claimed bounds on the singular values of $XX^\top$.
\end{proof}

\begin{thm}
\label{thm:ev-bound-gaussian-local}
Let $X \in \R^{p \times n}$ be a matrix whose columns are sampled i.i.d from a standard Gaussian distribution i.e. $\x_i \sim \cN(\vz,I)$. Then for any $\gamma > 0$, with probability at least $1 - \delta$, the matrix $X$ satisfies the SSC and SSS properties with constants
\begin{align*}
\Lambda_\gamma^{\text{Gauss}} &\leq \gamma n\br{1 + 3e\sqrt{6\log\frac{e}{\gamma}}} + \O{\sqrt{np + n\log\frac{1}{\delta}}}\\
\lambda_\gamma^{\text{Gauss}} &\geq n - (1-\gamma)n\br{1 + 3e\sqrt{6\log\frac{e}{1-\gamma}}} - \Om{\sqrt{np + n\log\frac{1}{\delta}}}.
\end{align*}
\end{thm}
\begin{proof}
For any fixed $S \in \cS_\gamma$, Lemma~\ref{lem:ev-bound-gaussian-global} guarantees the following bound
\[
s_{\max}(X_SX_S^\top) \leq \gamma n + (1 - 2\epsilon)^{-1}\sqrt{c\gamma np + c'\gamma n\log\frac{2}{\delta}}.
\]
Taking a union bound over $\cS_\gamma$ and noting that $\binom{n}{k} \leq \br{\frac{en}{k}}^k$ for all $1 \leq k \leq n$, gives us
\begin{align*}
\Lambda_\gamma &\leq \gamma n + (1 - 2\epsilon)^{-1}\sqrt{c\gamma np + c'\gamma^2n^2\log\frac{e}{\gamma} + c'\gamma n\log\frac{2}{\delta}}\\
			   &\leq \gamma n\br{1 + (1 - 2\epsilon)^{-1}\sqrt{c'\log\frac{e}{\gamma}}} + (1 - 2\epsilon)^{-1}\sqrt{c\gamma np + c'\gamma n\log\frac{2}{\delta}},
\end{align*}
which finishes the first bound after setting $\epsilon = 1/6$. For the second bound, we use the equality
\[
X_SX_S^\top = XX^\top - X_{\bar S}X_{\bar S}^\top,
\]
which provides the following bound for $\lambda_\gamma$
\[
\lambda_\gamma \geq s_{\min}(XX^\top) - \underset{T \in \cS_{1 - \gamma}}{\sup}X_TX_T^\top = s_{\min}(XX^\top) - \Lambda_{1-\gamma}.
\]
Using Lemma~\ref{lem:ev-bound-gaussian-global} to bound the first quantity and the first part of this theorem to bound the second quantity gives us, with probability at least $1 - \delta$,
\[
\lambda_\gamma \geq n - \gamma'n\br{1 + (1 - 2\epsilon)^{-1}\sqrt{c'\log\frac{e}{\gamma'}}} - (1 - 2\epsilon)^{-1}\br{1+\sqrt{\gamma'}}\sqrt{cnp + c'n\log\frac{2}{\delta}},
\]
where $\gamma' = 1 - \gamma$. This proves the second bound after setting $\epsilon = 1/6$.
\end{proof}

We now extend our analysis to the class of isotropic subGaussian distributions. We note that this analysis is without loss of generality since for non-isotropic sub-Gaussian distributions, we can simply use the fact that Theorem~\ref{thm:fcht} can admit whitened data for calculation of the SSC and SSS constants as we did for the case of non-isotropic Gaussian distributions.

\begin{lem}
\label{lem:ev-bound-subgaussian-global}
Let $X \in \R^{p \times n}$ be a matrix with columns sampled from some sub-Gaussian distribution with sub-Gaussian norm $K$ and covariance $\Sigma$. Then, for any $\delta > 0$, with probability at least $1 - \delta$, each of the following statements holds true:
\begin{align*}
	s_{\max}(XX^\top) &\leq \lambda_{\max}(\Sigma)\cdot n + C_K\cdot\sqrt{pn} + t\sqrt{n}\\
	s_{\min}(XX^\top) &\geq \lambda_{\min}(\Sigma)\cdot n - C_K\cdot\sqrt{pn} - t\sqrt{n},
\end{align*}
where $t = \sqrt{\frac{1}{c_K}\log\frac{2}{\delta}}$, and $c_K, C_K$ are absolute constants that depend only on the sub-Gaussian norm $K$ of the distribution.
\end{lem}
\begin{proof}
Since the singular values of a matrix are unchanged upon transposition, we shall prove the above statements for $X^\top$. The benefit of this is that we get to work with a matrix with independent rows, so that standard results can be applied. The proof technique used in \cite[Theorem 5.39]{Vershynin12} (see also Remark 5.40 (1) therein) can be used to establish the following result: with probability at least $1-\delta$, with $t$ set as mentioned in the theorem statement, we have
\[
\norm{\frac{1}{n}XX^\top - \Sigma} \leq C_K\sqrt\frac{p}{n} + \frac{t}{\sqrt n}
\]
This implies that for any $\v \in S^{p-1}$, we have
\[
\abs{\frac{1}{n}\norm{X^\top\v}_2^2 - \v^\top\Sigma\v} = \abs{\frac{1}{n}\v^\top XX^\top\v - \v^\top\Sigma\v} \leq \abs{\frac{1}{n}XX^\top\v - \Sigma\v} \leq C_K\sqrt\frac{p}{n} + \frac{t}{\sqrt n}.
\]
The results then follow from elementary manipulations and the fact that the singular values and eigenvalues of real symmetric matrices coincide.
\end{proof}

\begin{thm}
\label{thm:ev-bound-subgaussian-local}
Let $X \in \R^{p \times n}$ be a matrix with columns sampled from some sub-Gaussian distribution with sub-Gaussian norm $K$ and covariance $\Sigma$. Let $c_K, C_K$ and $t$ be fixed to values as required in Lemma~\ref{lem:ev-bound-subgaussian-global}. Note that $c_K$ and $C_K$ are absolute constants depend only on the sub-Gaussian norm $K$ of the distribution. Let $\gamma \in (0,1]$ be some fixed constant. Then, with  we have the following:
\[
\Lambda_\gamma^{\text{subGauss}(K,\Sigma)} \leq \br{\lambda_{\max}(\Sigma)\cdot\gamma + \sqrt{\frac{\gamma}{c_K}\log\frac{e}{\gamma}}}\cdot n + C_K\cdot\sqrt{\gamma pn} + t\sqrt{n}.
\]
Furthermore, fix any $\epsilon \in (0,1)$ and let $\gamma$ be a value in $(0,1)$ satisfying the following
\[
\gamma > 1 - \min\bc{\frac{\epsilon\cdot\lambda_{\min}(\Sigma)}{\lambda_{\max}(\Sigma)},\exp\br{1 + W_{-1}\br{-\frac{c_K\epsilon^2\cdot\lambda^2_{\min}(\Sigma)}{e}}}},
\]
where $W_{-1}(\cdot)$ is the lower branch of the real valued restriction of the Lambert W function. Then we have, with the same confidence,
\[
\lambda_\gamma^{\text{subGauss}(K,\Sigma)} \geq (1 - 2\epsilon)\cdot\lambda_{\min}(\Sigma)\cdot n - C_K\br{1 + \sqrt{1-\gamma}}\sqrt{pn} - 2t\sqrt{n}
\]
\end{thm}
\begin{proof}
The first result follows from an application of Lemma~\ref{lem:ev-bound-subgaussian-global}, a union bound over sets in $\cS_\gamma$, as well as the bound $\binom{n}{k} \leq \br{\frac{en}{k}}^k$ for all $1 \leq k \leq n$ which puts a bound on the number of sparse sets as $\log\abs{\cS_\gamma} \leq \gamma\cdot n\log\frac{e}{\gamma}$.

For the second result, we observe that $X_SX_S^\top = XX^\top - X_{\bar S}X_{\bar S}^\top$, so that $s_{\min}(X_SX_S^\top) \geq s_{\min}(XX^\top) - s_{\max}(X_{\bar S}X_{\bar S}^\top)$. This gives us
\[
\underset{S \in \cS_\gamma}{\inf}s_{\min}(X_SX_S^\top) \geq s_{\min}(XX^\top) - \underset{S \in \cS_{1-\gamma}}{\sup}s_{\max}(X_SX_S^\top).
\]
Using Lemma~\ref{lem:ev-bound-subgaussian-global} and the first part of this result gives us
\begin{align*}
\underset{S \in \cS_\gamma}{\inf}s_{\min}(X_SX_S^\top) \geq {} & \lambda_{\min}(\Sigma)\cdot n - C_K\cdot\sqrt{pn} - t\sqrt{n}\\
							& - {\br{\lambda_{\max}(\Sigma)(1-\gamma) + \sqrt{\frac{1-\gamma}{c_K}\log\frac{e}{1-\gamma}}} n - C_K\sqrt{(1-\gamma) pn} - t\sqrt{n}}\\
							= {} & \br{\lambda_{\min}(\Sigma) - \lambda_{\max}(\Sigma)(1-\gamma) - \sqrt{\frac{1-\gamma}{c_K}\log\frac{e}{1-\gamma}}}n\\
							& - C_K\br{1 + \sqrt{1-\gamma}}\sqrt{pn} - 2t\sqrt{n}\\
							\geq {} & (1-2\epsilon)\cdot\lambda_{\min}(\Sigma)\cdot n - C_K\br{1 + \sqrt{1-\gamma}}\sqrt{pn} - 2t\sqrt{n},
\end{align*}
where the last step follows from the assumptions on $\gamma$ and by noticing that it suffices to show the following two inequalities to establish the last step
\begin{enumerate}
	\item $\lambda_{\max}(\Sigma)(1-\gamma) \leq \epsilon\cdot\lambda_{\min}(\Sigma)$
	\item $(1-\gamma)\log\frac{e}{1-\gamma} \leq c_K\epsilon^2\cdot\lambda^2_{\min}(\Sigma)$
\end{enumerate}
The first part gives us the condition $\gamma > 1 - \frac{\epsilon\cdot\lambda_{\min}(\Sigma)}{\lambda_{\max}(\Sigma)}$ in a straightforward manner. For the second part, denote $v = c_K\epsilon^2\cdot\lambda^2_{\min}(\Sigma)$. Note that for $v \geq 1$, \emph{all} values of $\gamma \in (0,1]$ satisfy the inequality.

Otherwise we require the use of the Lambert W function (also known as the product logarithm function). This function ensures that its value $W(z)$ for any $z > -1/e$ satisfies $z = W(z)e^{W(z)}$. In our case, making a change of variable $(1-\gamma) = e^\eta$ gives us the inequality $(\eta - 1)e^{\eta - 1} \geq -v/e$. Note that since $v \leq 1$ in this case, $-v/e \in (-1/e,0)$ i.e. a valid value for the Lambert W function. However, $(-1/e,0)$ is also the region in which the Lambert W function is multi-valued. Taking the worse bound for $\gamma$ by choosing the lower branch $W_{-1}(\cdot)$ gives us the second condition $\gamma \geq 1 - \exp\br{1 + W_{-1}\br{-\frac{c_K\epsilon^2\cdot\lambda^2_{\min}(\Sigma)}{e}}}$.
\end{proof}
It is important to note that for any $-1/e \leq z < 0$, we have $\exp\br{1 + W_{-1}(z)} > 0$ which means that the bounds imposed on $\gamma$ by Theorem~\ref{thm:ev-bound-subgaussian-local} always allow a non-zero fraction of the data points to be corrupted in an adversarial manner. However, the exact value of that fraction depends, in a complicated manner, on the sub-Gaussian norm of the underlying distribution, as well as the condition number and the smallest eigenvalue of the second moment of the underlying distribution.

We also note that due to the generic nature of the previous analysis, which can handle the entire class of sub-Gaussian distributions, the bounds are not as explicitly stated in terms of universal constants as they are for the standard Gaussian design setting (Theorem~\ref{thm:ev-bound-gaussian-local}).

We now establish that for a wide family of random designs, the SRSC and SRSS properties are satisfied with high probability as well. For sake of simplicity, we will present our analysis for the standard Gaussian design. However, the results would readily extend to general Gaussian and sub-Gaussian designs using techniques similar to Theorem~\ref{thm:ev-bound-subgaussian-local}.

\begin{thm}
\label{thm:ev-bound-high-d-gaussian-local}
Let $X \in \R^{p \times n}$ be a matrix whose columns are sampled i.i.d from a standard Gaussian distribution i.e. $\x_i \sim \cN(\vz,I)$. Then for any $\gamma > 0$ and $s \leq p$, with probability at least $1 - \delta$, the matrix $X$ satisfies the SRSC and SRSS properties with constants
\begin{align*}
L_{(\gamma,s)}^{\text{Gauss}} &\leq \gamma n\br{1 + 3e\sqrt{6\log\frac{e}{\gamma}}} + \softO{\sqrt{ns+n\log\frac{1}{\delta}}}\\
\alpha_{(\gamma,s)}^{\text{Gauss}} &\geq n - (1-\gamma)n\br{1 + 3e\sqrt{6\log\frac{e}{1 - \gamma}}} - \softOm{\sqrt{ns+n\log\frac{1}{\delta}}}.
\end{align*}
\end{thm}
\begin{proof}
The proof of this theorem proceeds similarly to that of Theorem~\ref{thm:ev-bound-gaussian-local}. Hence, we simply point out the main differences. First, we shall establish, that for any $\epsilon > 0$, with probability at least $1-\delta$, $X$ satisfies the RSC and RSS properties at level $s$ with the following constants
\begin{align*}
L_s &\leq n + (1-2\epsilon)^{-1}\sqrt{bns + b'n\log\frac{2}{\delta}}\\
\alpha_s &\geq n - (1-2\epsilon)^{-1}\sqrt{bns + b'n\log\frac{2}{\delta}},
\end{align*}
where $b = 24e^2\log\frac{3ep}{\epsilon s}$ and $b' = 24e^2$. To do so we notice that the only change needed to be made would be in the application of the covering number argument. Instead of applying the union bound over an $\epsilon$-cover $\cC^{p-1}$ of $S^{p-1}$, we would only have to consider an $\epsilon$-cover $\cC^{p-1}_s$ of the set $S^{p-1}_s$ of all $s$-sparse unit vectors in $p$-dimensions. A straightforward calculation shows us that
\[
\abs{\cC^{p-1}_s} \leq \binom{p}{s}\br{1+\frac{2}{\epsilon}}^s \leq \br{\frac{3ep}{\epsilon s}}^s.
\]
Thus, setting $\mu^2 = b\cdot\frac{s}{n} + b'\cdot\frac{\log\frac{2}{\delta}}{n}$, where $b = 24e^2\log\frac{3ep}{\epsilon s}$ and $b' = 24e^2$, we get
\[
\Pr{\underset{\v \in \cC^{p-1}_s}\sup\abs{\norm{X\v}_2^2 - n} \geq \sqrt{bns + b'n\log\frac{2}{\delta}}} \leq \delta,
\]
which establishes the required RSC and RSS constants for $X$. Now, moving on to the SRSS constant, it follows simply by applying a union bound over all sets in $\cS_\gamma$ much like in Theorem~\ref{thm:ev-bound-gaussian-local}. One can then proceed to bound the SRSC constant in a similar manner.

We note that the nature of the SRSC and SRSS bounds indicate that our \alg-FC algorithm in the high dimensional sparse recovery setting has noise tolerance properties, characterized by the largest corruption index $\alpha$ that can be tolerated, identical to its low dimnensional counterpart - something that Theorem~\ref{thm:fcht-highd} states explicitly.
\end{proof}


\section{Supplementary Results}
\begin{clm}
\label{clm:supp-ord-stat}
Given any vector $\v \in \R^n$, let $\sigma \in S_n$ be defined as the permutation that orders elements of $\v$ in descending order of their magnitudes i.e. $\abs{v_{\sigma(1)}} \geq \abs{v_{\sigma(2)}} \geq \ldots \geq \abs{v_{\sigma(n)}}$. For any $0 < p \leq q \leq 1$, let $S_1 \in \cS_q$ be an arbitrary set of size $q\cdot n$ and $S_2 = \bc{\sigma(i): n-p\cdot n + 1 \leq i \leq n}$. Then we have $\norm{\v_{S_2}}_2^2 \leq \frac{p}{q} \norm{\v_{S_1}}_2^2 \leq \norm{\v_{S_1}}_2^2$.
\end{clm}
\begin{proof}
Let $S_3 = \bc{\sigma(i): n-q\cdot n + 1 \leq i \leq n}$ and $S_4 = \bc{\sigma(i): n-q\cdot n + 1 \leq i \leq n - p\cdot n}$. Clearly, we have $\norm{\v_{S_3}}_2^2 \leq \norm{\v_{S_1}}_2^2$ since $S_3$ contains the smallest $q\cdot n$ elements (by magnitude). Now we have $\norm{\v_{S_3}}_2^2 = \norm{\v_{S_2}}_2^2 + \norm{\v_{S_4}}_2^2$. Moreover, since each element of $S_4$ is larger in magnitude than every element of $S_2$, we have
\[
\frac{1}{\abs{S_4}}\norm{\v_{S_4}}_2^2 \geq \frac{1}{\abs{S_2}}\norm{\v_{S_2}}_2^2.
\]
This gives us
\[
\norm{\v_{S_2}}_2^2 = \norm{\v_{S_3}}_2^2 - \norm{\v_{S_4}}_2^2 \leq \norm{\v_{S_3}}_2^2 - \frac{\abs{S_4}}{\abs{S_2}}\norm{\v_{S_2}}_2^2,
\]
which upon simple manipulations, gives us the claimed result.
\end{proof}

\begin{lem}
\label{lem:chi-2-bernstein}
Let $Z$ be distributed according to the chi-squared distribution with $k$ degrees of freedom i.e. $Z \sim \chi^2(k)$. Then for all $t \geq 0$,
\[
\Pr{\abs{Z -k}\geq t} \leq 2\exp\br{-\min\bc{\frac{t^2}{24ke^2},\frac{t}{4\sqrt 3e}}}
\]
\end{lem}
\begin{proof}
This lemma requires a proof structure that traces several basic results in concentration inequalities for sub-exponential variables \cite[Lemma 5.5, 5.15, Proposition 5.17]{Vershynin12}. The purpose of performing this exercise is to explicate the constants involved so that a crisp bound can be provided on the corruption index that our algorithm can tolerate in the standard Gaussian design case.

We first begin by establishing the sub-exponential norm of a chi-squared random variable with a single degree of freedom. Let $X \sim \chi^2(1)$. Then using standard results on the moments of the standard normal distribution gives us, for all $p \geq 2$,
\[
(\Ebb{\abs{X}^p})^{1/p} = ((2p-1)!!)^{1/p} = \br{\frac{(2p)!}{2^pp!}}^{1/p} \leq \frac{\sqrt 3}{2}p
\]
Thus, the sub-exponential norm of $X$ is upper bounded by $\sqrt 3/2$. By applying the triangle inequality, we obtain, as a corollary, an upper bound on the sub-exponential norm of the centered random variable $Y = X - 1$ as $\norm{Y}_{\psi_1} \leq 2\norm{X}_{\psi_1} \leq \sqrt 3$.

Now we bound the moment generating function of the random variable $Y$. Noting that $\Ebb Y = 0$, we have, for any $\abs{\lambda}\leq \frac{1}{2\sqrt{3}e}$,
\[
\Ebb \exp(\lambda Y) = 1 + \sum_{q=2}^\infty\frac{\Ebb (\lambda Y)^q}{q!} \leq 1 + \sum_{q=2}^\infty\frac{(\sqrt 3|\lambda|q)^q}{q!} \leq 1 + \sum_{q=2}^\infty(\sqrt 3e|\lambda|)^q \leq 1 + 6e^2\lambda^2 \leq \exp(6e^2\lambda^2).
\]
Note that the second step uses the sub-exponentially of $Y$, the third step uses the fact that $q! \geq (q/e)^q$, and the fourth step uses the bound on $|\lambda|$. Now let $X_1,X_2,\ldots X_k$ be $k$ independent random variables distributed as $\chi^2(1)$. Then we have $Z \sim \sum_{i=1}^kX_i$. Using the exponential Markov's inequality, and the independence of the random variables $X_i$ gives us
\[
\Pr{Z - k\geq t} = \Pr{e^{\lambda(Z-k)} \geq e^{\lambda t}} \leq e^{-\lambda t}\Ebb e^{\lambda (Z-k)} = e^{-\lambda t}\prod_{i=1}^k\Ebb\exp(\lambda(X_i-1)).
\]
For any $|\lambda| \leq \frac{1}{2\sqrt 3e}$, the above bounds on the moment generating function give us
\[
\Pr{Z - k \geq t} \leq e^{-\lambda t}\prod_{i=1}^k\exp(6e^2\lambda^2) = \exp(-\lambda t + 6ke^2\lambda^2).
\]
Choosing $\lambda = \min\bc{\frac{1}{2\sqrt 3e},\frac{t}{12ke^2}}$, we get
\[
\Pr{Z - k \geq t} \leq \exp\br{-\min\bc{\frac{t^2}{24ke^2},\frac{t}{4\sqrt 3e}}}.
\]
Repeating this argument gives us the same bound for $\Pr{k - Z \geq t}$. This completes the proof.
\end{proof}
\clearpage
\section{Supplementary Experimental Results}

\begin{figure}[h!]
	\centering
	\hspace*{-1ex}
	\subfigure[\hspace*{-3ex}]{
		\includegraphics[scale=0.175]{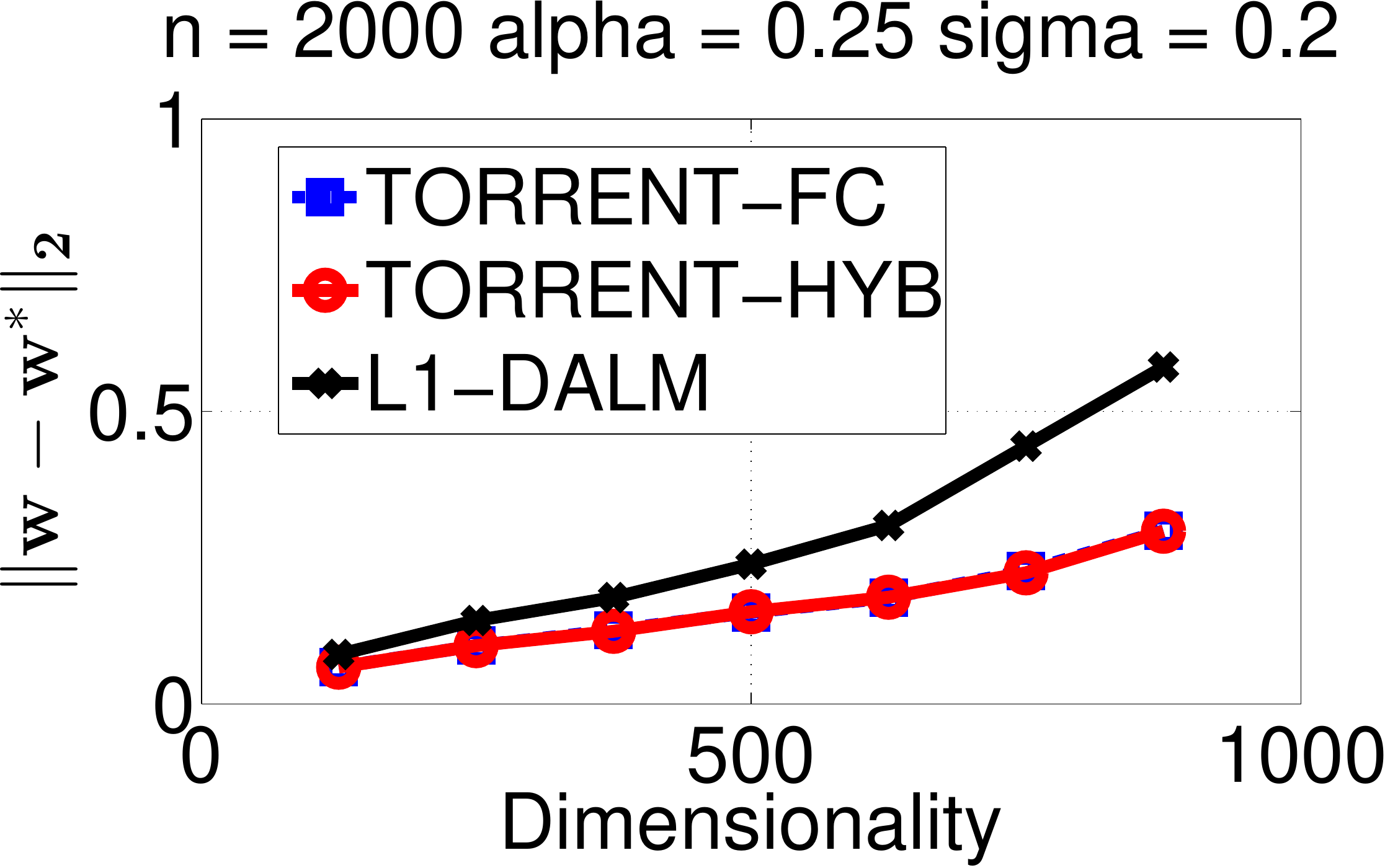}
	}\hspace*{-1ex}
	\subfigure[\hspace*{-3ex}]{
		\includegraphics[scale=0.175]{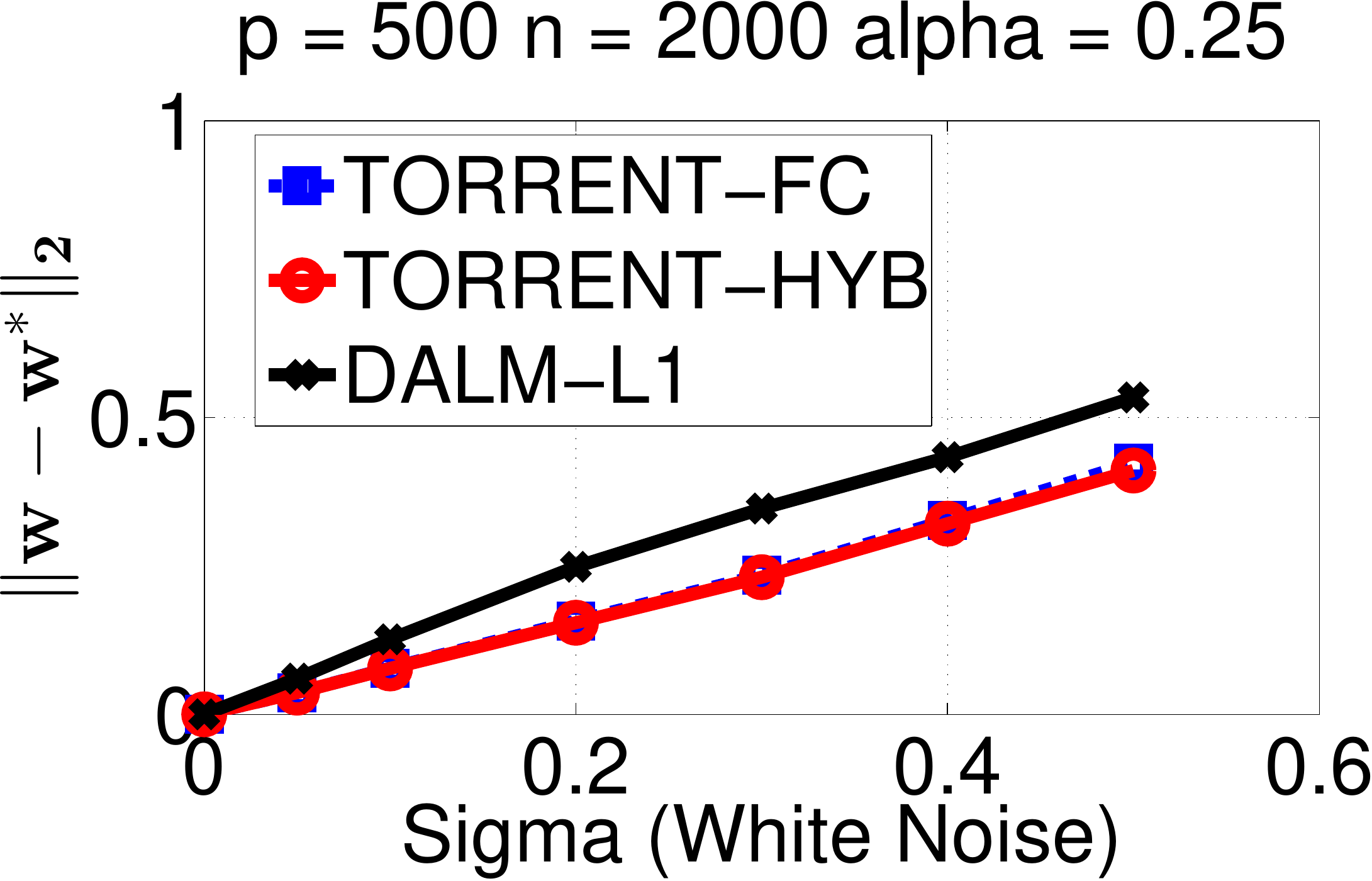}
	}\hspace*{-1ex}
	\subfigure[\hspace*{-3ex}]{
		\includegraphics[scale=0.175]{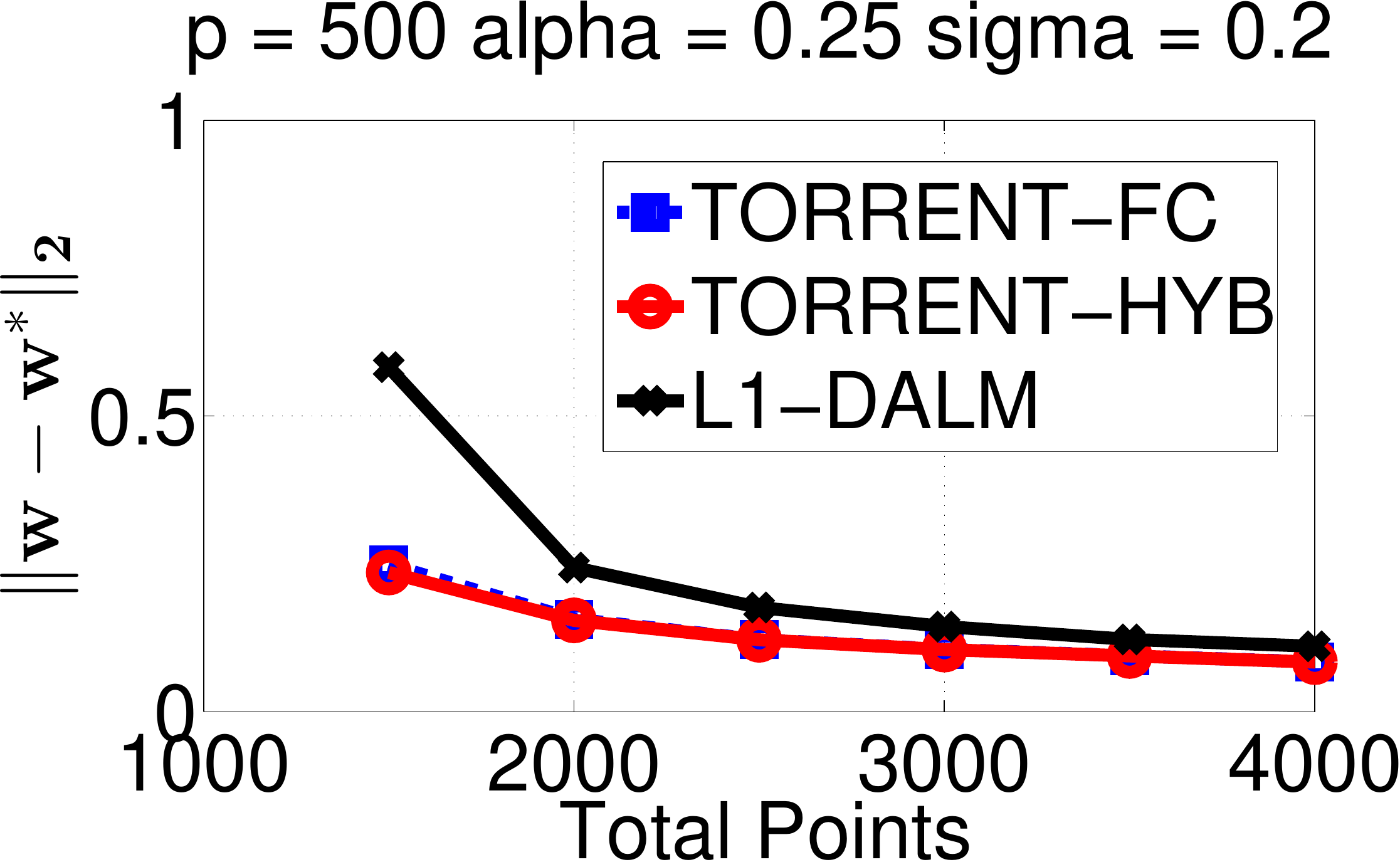}
	}\hspace*{-1ex}
	\subfigure[\hspace*{-3ex}]{
		\includegraphics[scale=0.175]{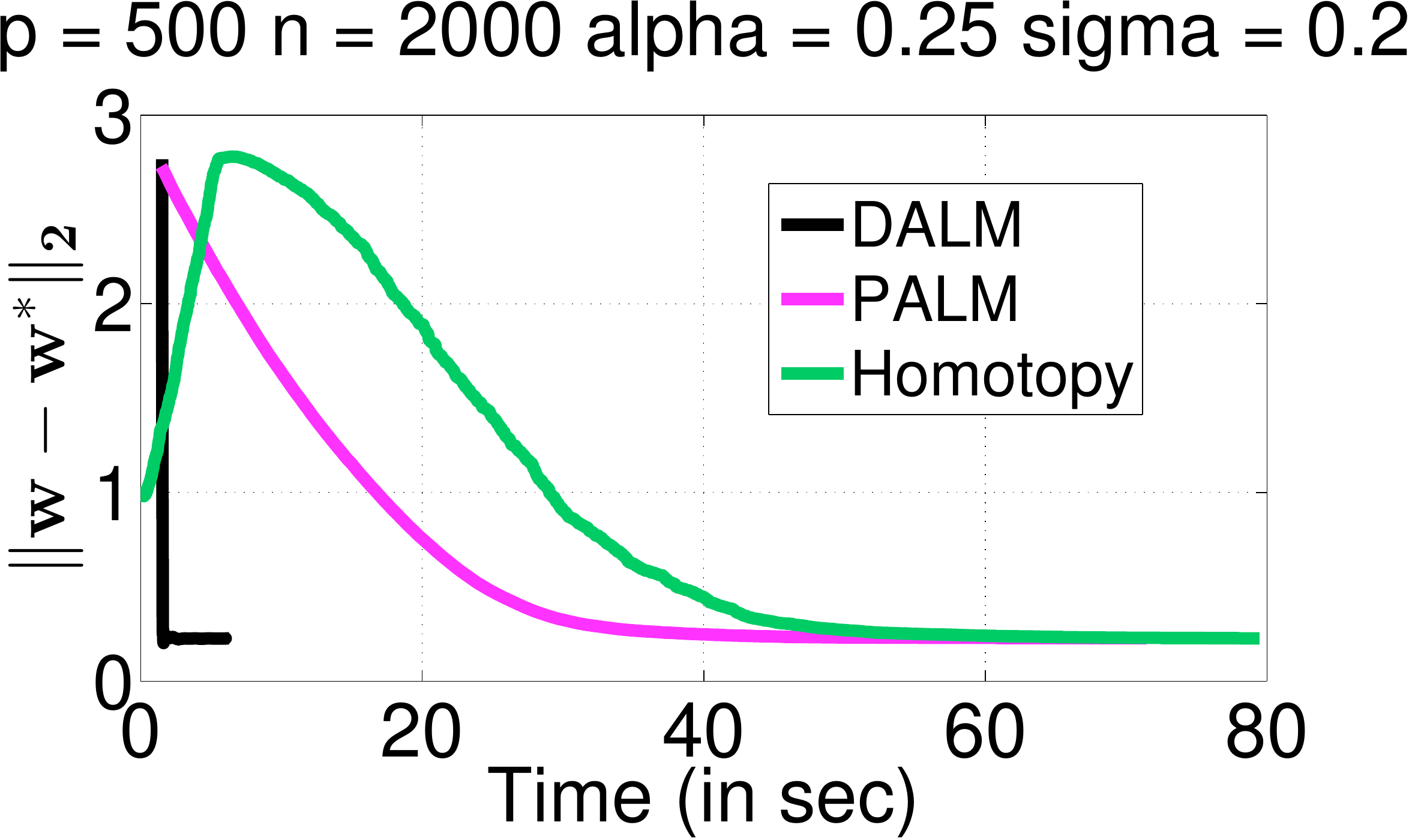}
	}
\hspace{2ex}
\caption{\small (a), (b), (c) Variation of recovery error with varying $p, \sigma$ and $n$. \alg was found to outperform DALM-$L_1$ in all these settings. (d) Recovery error as a function of runtime for various state-of-the-art $L_1$ solvers as indicated in \cite{YangGZSM12}.}
\label{fig:app}
\end{figure}
\end{document}